\documentclass{article}

\usepackage{microtype}
\usepackage{graphicx}
\usepackage{booktabs} % for professional tables

\usepackage{tkz-tab}
\usepackage{pgfplots}
\usepackage{verbatim}
\usetikzlibrary{positioning,decorations.pathreplacing,decorations.markings,calc,arrows,hobby,backgrounds,patterns}
\pgfplotsset{compat=1.7}
\usetikzlibrary{positioning, decorations.pathreplacing, decorations.markings, calc,arrows, pgfplots.groupplots, decorations.pathreplacing, decorations.markings, patterns}
\usetikzlibrary{shapes.geometric}
\usepackage{microtype}
\usepackage{graphicx}
\usepackage{adjustbox}
\usepackage{dsfont}
\usepackage{booktabs} % for professional tables

\usetikzlibrary{positioning, decorations.pathreplacing, decorations.markings, calc,arrows, pgfplots.groupplots, decorations.pathreplacing, decorations.markings, patterns}
\tikzset{
    axis break gap/.initial=0mm
}
\usepackage{hyperref}

\usepackage[accepted]{icml2024}

\usepackage{amsmath}
\usepackage{amssymb}
\usepackage{mathtools}
\usepackage{amsthm}
\usepackage{adjustbox}
\usepackage{tikz}
\usetikzlibrary{positioning,decorations.pathreplacing,decorations.markings,calc,arrows,hobby,backgrounds,patterns,calligraphy}
\usepackage{wrapfig}
\usepackage{subcaption}

\usepackage[capitalize,noabbrev]{cleveref}

\theoremstyle{plain}

\theoremstyle{definition}

\theoremstyle{remark}

\usepackage{thm-restate}

\usepackage{graphicx}
\usepackage{stackengine}

\DeclareFontFamily{U}{mathx}{\hyphenchar\font45}
\DeclareFontShape{U}{mathx}{m}{n}{%
<-6> mathx5
<6-7> mathx6
<7-8> mathx7
<8-9> mathx8
<9-10> mathx9
<10-12> mathx10
<12-> mathx12
}{}

\DeclareSymbolFont{mathx}{U}{mathx}{m}{n}
\DeclareFontSubstitution{U}{mathx}{m}{n}

\DeclareMathSymbol{\bigovoid}{\mathop}{mathx}{"EC}

\newcommand{\bigO}{\mathop{\stackinset{c}{-1.5pt}{c}{}{ \scalebox{0.8}{$\bigovoid$}}{\scalebox{0.8}{$\bigovoid$}}}}

\newcommand{\Conv}{%
  \mathop{\!\scalebox{2}{\raisebox{-0.2ex}{$\circledast$}}
  }
}
\newcommand{\conv}{%
  \mathop{\scalebox{1.5}{\raisebox{-0.1ex}{$\circledast$}}
  }
}
\usepackage{bm}

\newcommand{\sblue}[1]{\textcolor{blue}{\bm{#1}}}
\definecolor{gr}{RGB}{60,180,100}
\definecolor{bl}{RGB}{70,70,240}
\definecolor{sky}{RGB}{100,180,240}
\definecolor{yl}{RGB}{250,170,30}
\definecolor{or}{RGB}{200,140,80}
\definecolor{pp}{RGB}{200,150,240}
\definecolor{darkred}{RGB}{200,30,0}

\def\argmax{ \mathop{{\rm argmax}}}

\usepackage[textsize=tiny]{todonotes}

\icmltitlerunning{LayerMerge: Neural Network Depth Compression through Layer Pruning and Merging}

\pgfplotsset{compat=1.18}
\begin{document}

\twocolumn[
\icmltitle{LayerMerge: Neural Network Depth Compression \\through Layer Pruning and Merging}



\icmlsetsymbol{equal}{*}

\begin{icmlauthorlist}
\icmlauthor{Jinuk Kim}{snu,nprc}
\icmlauthor{Marwa El Halabi}{samsung} 
\icmlauthor{Mingi Ji}{google}
\icmlauthor{Hyun Oh Song}{snu,nprc}
\end{icmlauthorlist}

\icmlaffiliation{snu}{
Department of Computer Science and Engineering, Seoul National University}
\icmlaffiliation{samsung}{
Samsung - SAIT AI Lab, Montreal}
\icmlaffiliation{google}{
Google}
\icmlaffiliation{nprc}{
Neural Processing Research Center}

\icmlcorrespondingauthor{Hyun Oh Song}{hyunoh@snu.ac.kr} 
\icmlkeywords{Machine Learning, ICML}

\vskip 0.3in
]



\printAffiliationsAndNotice{}  % leave blank if no need to mention equal contribution

\begin{abstract}
Recent works show that reducing the number of layers in a convolutional neural network can enhance efficiency while maintaining the performance of the network.
Existing depth compression methods remove redundant non-linear activation functions and merge the consecutive convolution layers into a single layer.
However, these methods suffer from a critical drawback; the kernel size of the merged layers becomes larger, significantly undermining the latency reduction gained from reducing the depth of the network.
We show that this problem can be addressed by jointly pruning convolution layers and activation functions.
To this end, we propose \textit{LayerMerge}, a novel depth compression method that selects which activation layers and convolution layers to remove, to achieve a desired inference speed-up while minimizing performance loss.
Since the corresponding selection problem involves an exponential search space, we formulate a novel surrogate optimization problem and efficiently solve it via dynamic programming.
Empirical results demonstrate that our method consistently outperforms existing depth compression and layer pruning methods on various network architectures, both on image classification and generation tasks.
We release the code at \url{https://github.com/snu-mllab/LayerMerge}.
\end{abstract}

%This work focus on accelerating convolutional neural networks and diffusion probabilistic models by reducing the number of layers in the network.
%extending the search space and further optimizing to replace the subset of convolution layers to the identity layers.

\section{Introduction}

\newcommand\grid[3]{
    \def\p{#1}
    \def\xx{#2}
    \def\yy{#3}
  \draw[fill,\c,opacity=0.9] (\p) rectangle ++(\xx,\yy);
  \draw[gray!60] ($(\p) + (\xx/3,0)$) -- ($(\p) + (\xx/3,\yy)$);
  \draw[gray!60] ($(\p) + (\xx/3*2,0)$) -- ($(\p) + (\xx/3*2,\yy)$);
  \draw[gray!60] ($(\p) + (0,\yy/3)$) -- ($(\p) + (\xx,\yy/3)$);
  \draw[gray!60] ($(\p) + (0,\yy/3*2)$) -- ($(\p) + (\xx,\yy/3*2)$);
  \draw (\p) rectangle ++(\xx,\yy);
}
\newcommand\gridd[3]{
    \def\p{#1}
    \def\xx{#2}
    \def\yy{#3}
  \draw[fill,\c,opacity=0.9] (\p) rectangle ++(\xx,\yy);
  \draw[gray!60] ($(\p) + (\xx/5,0)$) -- ($(\p) + (\xx/5,\yy)$);
  \draw[gray!60] ($(\p) + (\xx/5*2,0)$) -- ($(\p) + (\xx/5*2,\yy)$);
  \draw[gray!60] ($(\p) + (\xx/5*3,0)$) -- ($(\p) + (\xx/5*3,\yy)$);
  \draw[gray!60] ($(\p) + (\xx/5*4,0)$) -- ($(\p) + (\xx/5*4,\yy)$);
  \draw[gray!60] ($(\p) + (0,\yy/5)$) -- ($(\p) + (\xx,\yy/5)$);
  \draw[gray!60] ($(\p) + (0,\yy/5*2)$) -- ($(\p) + (\xx,\yy/5*2)$);
  \draw[gray!60] ($(\p) + (0,\yy/5*3)$) -- ($(\p) + (\xx,\yy/5*3)$);
  \draw[gray!60] ($(\p) + (0,\yy/5*4)$) -- ($(\p) + (\xx,\yy/5*4)$);
  \draw (\p) rectangle ++(\xx,\yy);
}
\newcommand\griddd[3]{
    \def\p{#1}
    \def\xx{#2}
    \def\yy{#3}
  \draw[fill,\c,opacity=0.9] (\p) rectangle ++(\xx,\yy);
  \draw[gray!60] ($(\p) + (\xx/7,0)$) -- ($(\p) + (\xx/7,\yy)$);
  \draw[gray!60] ($(\p) + (\xx/7*2,0)$) -- ($(\p) + (\xx/7*2,\yy)$);
  \draw[gray!60] ($(\p) + (\xx/7*3,0)$) -- ($(\p) + (\xx/7*3,\yy)$);
  \draw[gray!60] ($(\p) + (\xx/7*4,0)$) -- ($(\p) + (\xx/7*4,\yy)$);
  \draw[gray!60] ($(\p) + (\xx/7*5,0)$) -- ($(\p) + (\xx/7*5,\yy)$);
  \draw[gray!60] ($(\p) + (\xx/7*6,0)$) -- ($(\p) + (\xx/7*6,\yy)$);
  \draw[gray!60] ($(\p) + (0,\yy/7)$) -- ($(\p) + (\xx,\yy/7)$);
  \draw[gray!60] ($(\p) + (0,\yy/7*2)$) -- ($(\p) + (\xx,\yy/7*2)$);
  \draw[gray!60] ($(\p) + (0,\yy/7*3)$) -- ($(\p) + (\xx,\yy/7*3)$);
  \draw[gray!60] ($(\p) + (0,\yy/7*4)$) -- ($(\p) + (\xx,\yy/7*4)$);
  \draw[gray!60] ($(\p) + (0,\yy/7*5)$) -- ($(\p) + (\xx,\yy/7*5)$);
  \draw[gray!60] ($(\p) + (0,\yy/7*6)$) -- ($(\p) + (\xx,\yy/7*6)$);
  \draw (\p) rectangle ++(\xx,\yy);
}

\newcommand\fm[5]{
    \def\nin{#1}
    \def\convid{#2}
    \def\xx{#3}
    \def\yy{#4}
    \def\ccc{#5}
        \foreach \i in {0,1,...,\nin} {
        \begin{scope}[  % Lower layer
            yshift={1*\channelwidth*\i},every node/.append style={
                 yslant=0,xslant=1,rotate=\rot},yslant=0,xslant=1,rotate=\rot
              ]              
              \coordinate (A\convid\i) at (\xx,\yy);
              \coordinate (DA\convid\i) at ($(\xx,\yy)+(\fmw,\fmh)$);
        \end{scope}
        }
    \draw[\ccc] (A\convid0) -- (A\convid\nin);
    \draw[\ccc] ($(A\convid0) + (\fmw,0)$) -- ($(A\convid\nin) + (\fmw,0)$);
    \draw[\ccc] (DA\convid0) -- (DA\convid\nin);
        \foreach \i in {0,1,...,\nin} {
        \begin{scope}[  % Lower layer
            yshift={1.5*\channelwidth*\i},every node/.append style={
                 yslant=0,xslant=1,rotate=\rot},yslant=0,xslant=1,rotate=\rot
              ]              
              \draw[fill,\ccc,opacity=0.8] (A\convid\i) rectangle ++(\fmw,\fmh);
              \ifnum\i=0
              \draw[red!40] (A\convid\i) rectangle ++(\fmw,\fmh);
              \else
            \ifnum\i=\nin
              \draw (A\convid\i) rectangle ++(\fmw,\fmh);            
            \else 
              \draw[red!40] (A\convid\i) rectangle ++(\fmw,\fmh);           
            \fi
            \fi
        \end{scope}
        }
    \draw[fill,\ccc,opacity=0.7] (A\convid0) -- (A\convid\nin) -- ($(A\convid\nin)+(\fmw,0)$) -- ($(A\convid0)+(\fmw,0)$) --cycle; 
    \draw[fill,\ccc,opacity=0.7] (DA\convid0) -- (DA\convid\nin) -- ($(A\convid\nin)+(\fmw,0)$) -- ($(A\convid0)+(\fmw,0)$) --cycle; 
    \draw[black] (DA\convid0) -- (DA\convid\nin) -- ($(A\convid\nin)+(\fmw,0)$) -- ($(A\convid0)+(\fmw,0)$) --cycle; 
    \draw[black] (A\convid0) -- (A\convid\nin) -- ($(A\convid\nin)+(\fmw,0)$) -- ($(A\convid0)+(\fmw,0)$) --cycle; 
    \foreach \i in {0,1,...,\nin} {
        \begin{scope}[  % Lower layer
            yshift={1.5*\channelwidth*\i},every node/.append style={
                 yslant=0,xslant=1,rotate=\rot},yslant=0,xslant=1,rotate=\rot
              ]              
              \draw[red!80!black!60] ($(A\convid\i)+(\fmw,0)$) -- ($(A\convid\i)+(\fmw,\fmh)$);
              \draw[red!80!black!60] (A\convid\i) -- ($(A\convid\i)+(\fmw,0.0)$);
        \end{scope}
        }
        \begin{scope}[  % Lower layer
            yshift={-100},every node/.append style={
                 yslant=0,xslant=1,rotate=\rot},yslant=0,xslant=1,rotate=\rot
              ]              
              \draw[black] ($(A\convid\nin)+(\fmw,0)$) -- ($(A\convid\nin)+(\fmw,\fmh)$);
              \draw[black] (A\convid\nin) -- ($(A\convid\nin)+(\fmw,0.0)$);
              \draw[black] ($(A\convid0)+(\fmw,0)$) -- ($(A\convid0)+(\fmw,\fmh)$);
              \draw[black] (A\convid0) -- ($(A\convid0)+(\fmw,0.0)$);
        \end{scope}
}

\newcommand\ggfm[5]{
    \def\nin{#1}
    \def\convid{#2}
    \def\xx{#3}
    \def\yy{#4}
    \def\margin{#5}
    \def\ccc{red!20}
        \foreach \i in {0,1,...,\nin} {
        \begin{scope}[  % Lower layer
            yshift={\margin + 1*\channelwidth*\i},every node/.append style={
                 yslant=0,xslant=1,rotate=\rot},yslant=0,xslant=1,rotate=\rot
              ]              
              \coordinate (A\convid\i) at (\xx,\yy);
              \coordinate (DA\convid\i) at ($(\xx,\yy)+(1,0.7)$);
        \end{scope}
        }
    \draw[\ccc] (A\convid0) -- (A\convid\nin);
    \draw[\ccc] ($(A\convid0) + (1,0)$) -- ($(A\convid\nin) + (1,0)$);
    \draw[\ccc] (DA\convid0) -- (DA\convid\nin);
        \foreach \i in {0,1,...,\nin} {
        \begin{scope}[  % Lower layer
            yshift={\margin + 1.5*\channelwidth*\i},every node/.append style={
                 yslant=0,xslant=1,rotate=\rot},yslant=0,xslant=1,rotate=\rot
              ]              
              \draw[fill,\ccc,opacity=0.7] (A\convid\i) rectangle ++(1,0.7);
              \ifnum\i=0
              \draw[red!40] (A\convid\i) rectangle ++(1,0.7);
              \else
            \ifnum\i=\nin
              \draw (A\convid\i) rectangle ++(1,0.7);            
            \else 
              \draw[red!40] (A\convid\i) rectangle ++(1,0.7);           
            \fi
            \fi
        \end{scope}
        }
    \draw[fill,\ccc,opacity=0.7] (A\convid0) -- (A\convid\nin) -- ($(A\convid\nin)+(1,0)$) -- ($(A\convid0)+(1,0)$) --cycle; 
    \draw[fill,\ccc,opacity=0.7] (DA\convid0) -- (DA\convid\nin) -- ($(A\convid\nin)+(1,0)$) -- ($(A\convid0)+(1,0)$) --cycle; 
    \draw[black] (DA\convid0) -- (DA\convid\nin) -- ($(A\convid\nin)+(1,0)$) -- ($(A\convid0)+(1,0)$) --cycle; 
    \draw[black] (A\convid0) -- (A\convid\nin) -- ($(A\convid\nin)+(1,0)$) -- ($(A\convid0)+(1,0)$) --cycle; 
    \foreach \i in {0,1,...,\nin} {
        \begin{scope}[  % Lower layer
            yshift={\margin + 1.5*\channelwidth*\i},every node/.append style={
                 yslant=0,xslant=1,rotate=\rot},yslant=0,xslant=1,rotate=\rot
              ]              
              %\draw[green!80!black!100] ($(A\convid\i)+(1,0)$) -- ($(A\convid\i)+(1,0.7)$);
              \draw[red!80!black!60] ($(A\convid\i)+(1,0)$) -- ($(A\convid\i)+(1,0.7)$);
              %\draw[green!80!black!100] (A\convid\i) -- ($(A\convid\i)+(1,0.0)$);
              \draw[red!80!black!60] (A\convid\i) -- ($(A\convid\i)+(1,0.0)$);
        \end{scope}
        }
        \begin{scope}[  % Lower layer
            yshift={\margin},every node/.append style={
                 yslant=0,xslant=1,rotate=\rot},yslant=0,xslant=1,rotate=\rot
              ]              
              \draw[black] ($(A\convid\nin)+(1,0)$) -- ($(A\convid\nin)+(1,0.7)$);
              \draw[black] (A\convid\nin) -- ($(A\convid\nin)+(1,0.0)$);
              \draw[black] ($(A\convid0)+(1,0)$) -- ($(A\convid0)+(1,0.7)$);
              \draw[black] (A\convid0) -- ($(A\convid0)+(1,0.0)$);
        \end{scope}
}
\newcommand\gfm[5]{
    \def\nin{#1}
    \def\convid{#2}
    \def\xx{#3}
    \def\yy{#4}
    \def\margin{#5}
    \def\ccc{green!20}
        \foreach \i in {0,1,...,\nin} {
        \begin{scope}[  % Lower layer
            yshift={\margin + 1*\channelwidth*\i},every node/.append style={
                 yslant=0,xslant=1,rotate=\rot},yslant=0,xslant=1,rotate=\rot
              ]              
              \coordinate (A\convid\i) at (\xx,\yy);
              \coordinate (DA\convid\i) at ($(\xx,\yy)+(1,0.7)$);
        \end{scope}
        }
    \draw[\ccc] (A\convid0) -- (A\convid\nin);
    \draw[\ccc] ($(A\convid0) + (1,0)$) -- ($(A\convid\nin) + (1,0)$);
    \draw[\ccc] (DA\convid0) -- (DA\convid\nin);
        \foreach \i in {0,1,...,\nin} {
        \begin{scope}[  % Lower layer
            yshift={\margin + 1.5*\channelwidth*\i},every node/.append style={
                 yslant=0,xslant=1,rotate=\rot},yslant=0,xslant=1,rotate=\rot
              ]              
              \draw[fill,\ccc,opacity=0.7] (A\convid\i) rectangle ++(1,0.7);
              \ifnum\i=0
              %\draw[green!80!black!40] (A\convid\i) rectangle ++(1,0.7);
              \draw[green!80!black!40] (A\convid\i) rectangle ++(1,0.7);
              \else
            \ifnum\i=\nin
              \draw (A\convid\i) rectangle ++(1,0.7);            
            \else 
              %\draw[green!80!black!40] (A\convid\i) rectangle ++(1,0.7);           
              \draw[green!80!black!40] (A\convid\i) rectangle ++(1,0.7);           
            \fi
            \fi
        \end{scope}
        }
    \draw[fill,\ccc,opacity=0.7] (A\convid0) -- (A\convid\nin) -- ($(A\convid\nin)+(1,0)$) -- ($(A\convid0)+(1,0)$) --cycle; 
    \draw[fill,\ccc,opacity=0.7] (DA\convid0) -- (DA\convid\nin) -- ($(A\convid\nin)+(1,0)$) -- ($(A\convid0)+(1,0)$) --cycle; 
    \draw[black] (DA\convid0) -- (DA\convid\nin) -- ($(A\convid\nin)+(1,0)$) -- ($(A\convid0)+(1,0)$) --cycle; 
    \draw[black] (A\convid0) -- (A\convid\nin) -- ($(A\convid\nin)+(1,0)$) -- ($(A\convid0)+(1,0)$) --cycle; 
    \foreach \i in {0,1,...,\nin} {
        \begin{scope}[  % Lower layer
            yshift={\margin + 1.5*\channelwidth*\i},every node/.append style={
                 yslant=0,xslant=1,rotate=\rot},yslant=0,xslant=1,rotate=\rot
              ]              
              %\draw[green!80!black!100] ($(A\convid\i)+(1,0)$) -- ($(A\convid\i)+(1,0.7)$);
              \draw[green!80!black!100] ($(A\convid\i)+(1,0)$) -- ($(A\convid\i)+(1,0.7)$);
              %\draw[green!80!black!100] (A\convid\i) -- ($(A\convid\i)+(1,0.0)$);
              \draw[green!80!black!100] (A\convid\i) -- ($(A\convid\i)+(1,0.0)$);
        \end{scope}
        }
        \begin{scope}[  % Lower layer
            yshift={\margin},every node/.append style={
                 yslant=0,xslant=1,rotate=\rot},yslant=0,xslant=1,rotate=\rot
              ]              
              \draw[black] ($(A\convid\nin)+(1,0)$) -- ($(A\convid\nin)+(1,0.7)$);
              \draw[black] (A\convid\nin) -- ($(A\convid\nin)+(1,0.0)$);
              \draw[black] ($(A\convid0)+(1,0)$) -- ($(A\convid0)+(1,0.7)$);
              \draw[black] (A\convid0) -- ($(A\convid0)+(1,0.0)$);
        \end{scope}
}
\newcommand\bfm[5]{
    \def\nin{#1}
    \def\convid{#2}
    \def\xx{#3}
    \def\yy{#4}
    \def\ccc{#5}
        \foreach \i in {0,1,...,\nin} {
        \begin{scope}[  % Lower layer
            yshift={-140+1*\channelwidth*\i},every node/.append style={
                 yslant=0,xslant=1,rotate=\rot},yslant=0,xslant=1,rotate=\rot
              ]              
              \coordinate (BA\convid\i) at (\xx,\yy);
              \coordinate (DBA\convid\i) at ($(\xx,\yy)+(1,0.7)$);
        \end{scope}
        }
    \draw[\ccc] (BA\convid0) -- (BA\convid\nin);
    \draw[\ccc] ($(BA\convid0) + (1,0)$) -- ($(BA\convid\nin) + (1,0)$);
    \draw[\ccc] (DBA\convid0) -- (DBA\convid\nin);
        \foreach \i in {0,1,...,\nin} {
        \begin{scope}[  % Lower layer
            yshift={-100+1.5*\channelwidth*\i},every node/.append style={
                 yslant=0,xslant=1,rotate=\rot},yslant=0,xslant=1,rotate=\rot
              ]              
              \draw[fill,\ccc,opacity=0.8] (BA\convid\i) rectangle ++(1,0.7);
              \ifnum\i=0
              \draw[red!40] (BA\convid\i) rectangle ++(1,0.7);
              \else
            \ifnum\i=\nin
              \draw (BA\convid\i) rectangle ++(1,0.7);            
            \else 
              \draw[red!40] (BA\convid\i) rectangle ++(1,0.7);           
            \fi
            \fi
        \end{scope}
        }
    \draw[fill,\ccc,opacity=0.7] (BA\convid0) -- (BA\convid\nin) -- ($(BA\convid\nin)+(1,0)$) -- ($(BA\convid0)+(1,0)$) --cycle; 
    \draw[fill,\ccc,opacity=0.7] (DBA\convid0) -- (DBA\convid\nin) -- ($(BA\convid\nin)+(1,0)$) -- ($(BA\convid0)+(1,0)$) --cycle; 
    \draw[black] (DBA\convid0) -- (DBA\convid\nin) -- ($(BA\convid\nin)+(1,0)$) -- ($(BA\convid0)+(1,0)$) --cycle; 
    \draw[black] (BA\convid0) -- (BA\convid\nin) -- ($(BA\convid\nin)+(1,0)$) -- ($(BA\convid0)+(1,0)$) --cycle; 
    \foreach \i in {0,1,...,\nin} {
        \begin{scope}[  % Lower layer
            yshift={-100+1.5*\channelwidth*\i},every node/.append style={
                 yslant=0,xslant=1,rotate=\rot},yslant=0,xslant=1,rotate=\rot
              ]              
              \draw[red!80!black!60] ($(BA\convid\i)+(1,0)$) -- ($(BA\convid\i)+(1,0.7)$);
              \draw[red!80!black!60] (BA\convid\i) -- ($(BA\convid\i)+(1,0.0)$);
        \end{scope}
        }
        \begin{scope}[  % Lower layer
            yshift={-100},every node/.append style={
                 yslant=0,xslant=1,rotate=\rot},yslant=0,xslant=1,rotate=\rot
              ]              
              \draw[black] ($(BA\convid\nin)+(1,0)$) -- ($(BA\convid\nin)+(1,0.7)$);
              \draw[black] (BA\convid\nin) -- ($(BA\convid\nin)+(1,0.0)$);
              \draw[black] ($(BA\convid0)+(1,0)$) -- ($(BA\convid0)+(1,0.7)$);
              \draw[black] (BA\convid0) -- ($(BA\convid0)+(1,0.0)$);
        \end{scope}
}

\newcommand\ffm[5]{
    \def\nin{#1}
    \def\convid{#2}
    \def\xx{#3}
    \def\yy{#4}
    \def\ccc{#5}
        \foreach \i in {0,1,...,\nin} {
        \begin{scope}[  % Lower layer
            yshift={-200+1*\channelwidth*\i},every node/.append style={
                 yslant=0,xslant=1,rotate=\rot},yslant=0,xslant=1,rotate=\rot
              ]              
              \coordinate (BA\convid\i) at (\xx,\yy);
              \coordinate (DBA\convid\i) at ($(\xx,\yy)+(1,0.7)$);
        \end{scope}
        }
    \draw[\ccc] (BA\convid0) -- (BA\convid\nin);
    \draw[\ccc] ($(BA\convid0) + (1,0)$) -- ($(BA\convid\nin) + (1,0)$);
    \draw[\ccc] (DBA\convid0) -- (DBA\convid\nin);
        \foreach \i in {0,1,...,\nin} {
        \begin{scope}[  % Lower layer
            yshift={-200+1.5*\channelwidth*\i},every node/.append style={
                 yslant=0,xslant=1,rotate=\rot},yslant=0,xslant=1,rotate=\rot
              ]              
              \draw[fill,\ccc,opacity=0.8] (BA\convid\i) rectangle ++(1,0.7);
              \ifnum\i=0
              \draw[red!40] (BA\convid\i) rectangle ++(1,0.7);
              \else
            \ifnum\i=\nin
              \draw (BA\convid\i) rectangle ++(1,0.7);            
            \else 
              \draw[red!40] (BA\convid\i) rectangle ++(1,0.7);           
            \fi
            \fi
        \end{scope}
        }
    \draw[fill,\ccc,opacity=0.7] (BA\convid0) -- (BA\convid\nin) -- ($(BA\convid\nin)+(1,0)$) -- ($(BA\convid0)+(1,0)$) --cycle; 
    \draw[fill,\ccc,opacity=0.7] (DBA\convid0) -- (DBA\convid\nin) -- ($(BA\convid\nin)+(1,0)$) -- ($(BA\convid0)+(1,0)$) --cycle; 
    \draw[black] (DBA\convid0) -- (DBA\convid\nin) -- ($(BA\convid\nin)+(1,0)$) -- ($(BA\convid0)+(1,0)$) --cycle; 
    \draw[black] (BA\convid0) -- (BA\convid\nin) -- ($(BA\convid\nin)+(1,0)$) -- ($(BA\convid0)+(1,0)$) --cycle; 
    \foreach \i in {0,1,...,\nin} {
        \begin{scope}[  % Lower layer
            yshift={-200+1.5*\channelwidth*\i},every node/.append style={
                 yslant=0,xslant=1,rotate=\rot},yslant=0,xslant=1,rotate=\rot
              ]              
              \draw[red!80!black!60] ($(BA\convid\i)+(1,0)$) -- ($(BA\convid\i)+(1,0.7)$);
              \draw[red!80!black!60] (BA\convid\i) -- ($(BA\convid\i)+(1,0.0)$);
        \end{scope}
        }
        \begin{scope}[  % Lower layer
            yshift={-200},every node/.append style={
                 yslant=0,xslant=1,rotate=\rot},yslant=0,xslant=1,rotate=\rot
              ]              
              \draw[black] ($(BA\convid\nin)+(1,0)$) -- ($(BA\convid\nin)+(1,0.7)$);
              \draw[black] (BA\convid\nin) -- ($(BA\convid\nin)+(1,0.0)$);
              \draw[black] ($(BA\convid0)+(1,0)$) -- ($(BA\convid0)+(1,0.7)$);
              \draw[black] (BA\convid0) -- ($(BA\convid0)+(1,0.0)$);
        \end{scope}
}
\newcommand\bbconv[4]{
    \def\nin{#1}
    \def\convid{#2}
    \def\xx{#3}
    \def\yy{#4}
    \foreach \i in {0,1,...,\nin} {
        \begin{scope}[  % Lower layer
            yshift={-20+\channelwidth*\i},every node/.append style={
                 yslant=0,xslant=1,rotate=\rot},yslant=0,xslant=1,rotate=\rot
              ]              
                \coordinate (B\convid\i) at ($(\xx,\yy)+(0,0)$);
                \coordinate (DB\convid\i) at ($(\xx,\yy)+(0.5,0.35)$);
              \grid{B\convid\i}{0.5}{0.35}
        \end{scope}
    }

    \draw[fill,gray!10,opacity=0.8] (B\convid0) -- (B\convid\nin) -- ($(B\convid\nin)+(0.5,0)$) -- ($(B\convid0)+(0.5,0)$) --cycle; 
    \draw[fill,gray!10,opacity=0.8] (DB\convid0) -- (DB\convid\nin) -- ($(B\convid\nin)+(0.5,0)$) -- ($(B\convid0)+(0.5,0)$) --cycle; 
    \foreach \i in {0,1,...,\nin} {
        \begin{scope}[  % Lower layer
            yshift={\channelwidth*\i},every node/.append style={
                 yslant=0,xslant=1,rotate=\rot},yslant=0,xslant=1,rotate=\rot
              ]              
            \draw[gray!80] (B\convid\i) -- ($(B\convid\i)+(0.5,0)$);
            \draw[gray!80] (DB\convid\i) -- ($(B\convid\i)+(0.5,0)$);
        \end{scope}
    }
    \draw[black] (B\convid0) -- (B\convid\nin) -- ($(B\convid\nin)+(0.5,0)$) -- ($(B\convid0)+(0.5,0)$) --cycle; 
    \draw[black] (DB\convid0) -- (DB\convid\nin) -- ($(B\convid\nin)+(0.5,0)$) -- ($(B\convid0)+(0.5,0)$) --cycle;

}

\newcommand\bbbconv[4]{
    \def\nin{#1}
    \def\convid{#2}
    \def\xx{#3}
    \def\yy{#4}
    \foreach \i in {0,1,...,\nin} {
        \begin{scope}[  % Lower layer
            yshift={-230+\channelwidth*\i},every node/.append style={
                 yslant=0,xslant=1,rotate=\rot},yslant=0,xslant=1,rotate=\rot
              ]              
                \coordinate (B\convid\i) at ($(\xx,\yy)+(0,0)$);
                \coordinate (DB\convid\i) at ($(\xx,\yy)+(0.5,0.35)$);
              \grid{B\convid\i}{0.5}{0.35}
        \end{scope}
    }

    \draw[fill,gray!10,opacity=0.8] (B\convid0) -- (B\convid\nin) -- ($(B\convid\nin)+(0.5,0)$) -- ($(B\convid0)+(0.5,0)$) --cycle; 
    \draw[fill,gray!10,opacity=0.8] (DB\convid0) -- (DB\convid\nin) -- ($(B\convid\nin)+(0.5,0)$) -- ($(B\convid0)+(0.5,0)$) --cycle; 
    \foreach \i in {0,1,...,\nin} {
        \begin{scope}[  % Lower layer
            yshift={100+\channelwidth*\i},every node/.append style={
                 yslant=0,xslant=1,rotate=\rot},yslant=0,xslant=1,rotate=\rot
              ]              
            \draw[gray!80] (B\convid\i) -- ($(B\convid\i)+(0.5,0)$);
            \draw[gray!80] (DB\convid\i) -- ($(B\convid\i)+(0.5,0)$);
        \end{scope}
    }
    \draw[black] (B\convid0) -- (B\convid\nin) -- ($(B\convid\nin)+(0.5,0)$) -- ($(B\convid0)+(0.5,0)$) --cycle; 
    \draw[black] (DB\convid0) -- (DB\convid\nin) -- ($(B\convid\nin)+(0.5,0)$) -- ($(B\convid0)+(0.5,0)$) --cycle;
    % \draw[gray!60] ($(B\convid0) + (0.5/3,0)$) -- ($(B\convid\nin) + (0.5/3,0)$) ;
    % \draw[gray!60] ($(B\convid0) + (0.5/3*2,0)$) -- ($(B\convid\nin) + (0.5/3*2,0)$) ;

    % \draw[gray!60] ($(DB\convid0)!0.333333!(B\convid0) + (0.5/3,0)$) -- ($(DB\convid\nin)!0.333333!(B\convid\nin) + (0.5/3,0)$);
    % \draw[gray!60] ($(DB\convid0)!0.666667!(B\convid0) + (0.5/3*2,0)$) -- ($(DB\convid\nin)!0.666667!(B\convid\nin) + (0.5/3*2,0)$);
}
\newcommand\bconv[4]{
    \def\nin{#1}
    \def\convid{#2}
    \def\xx{#3}
    \def\yy{#4}
    \foreach \i in {0,1,...,\nin} {
        \begin{scope}[  % Lower layer
            yshift={-140+\channelwidth*\i},every node/.append style={
                 yslant=0,xslant=1,rotate=\rot},yslant=0,xslant=1,rotate=\rot
              ]              
                \coordinate (EB\convid\i) at ($(\xx,\yy)+(0,0)$);
                \coordinate (DEB\convid\i) at ($(\xx,\yy)+(0.5,0.35)$);
              \grid{EB\convid\i}{0.5}{0.35}
        \end{scope}
    }

    \draw[fill,gray!10,opacity=0.8] (EB\convid0) -- (EB\convid\nin) -- ($(EB\convid\nin)+(0.5,0)$) -- ($(EB\convid0)+(0.5,0)$) --cycle; 
    \draw[fill,gray!10,opacity=0.8] (DEB\convid0) -- (DEB\convid\nin) -- ($(EB\convid\nin)+(0.5,0)$) -- ($(EB\convid0)+(0.5,0)$) --cycle; 
    \foreach \i in {0,1,...,\nin} {
        \begin{scope}[  % Lower layer
            yshift={-100+\channelwidth*\i},every node/.append style={
                 yslant=0,xslant=1,rotate=\rot},yslant=0,xslant=1,rotate=\rot
              ]              
            \draw[gray!80] (EB\convid\i) -- ($(EB\convid\i)+(0.5,0)$);
            \draw[gray!80] (DEB\convid\i) -- ($(EB\convid\i)+(0.5,0)$);
        \end{scope}
    }
    \draw[black] (EB\convid0) -- (EB\convid\nin) -- ($(EB\convid\nin)+(0.5,0)$) -- ($(EB\convid0)+(0.5,0)$) --cycle; 
    \draw[black] (DEB\convid0) -- (DEB\convid\nin) -- ($(EB\convid\nin)+(0.5,0)$) -- ($(EB\convid0)+(0.5,0)$) --cycle; 
}

\newcommand\bconvv[4]{
    \def\nin{#1}
    \def\convid{#2}
    \def\xx{#3}
    \def\yy{#4}
    \foreach \i in {0,1,...,\nin} {
        \begin{scope}[  % Lower layer
            yshift={-110+\channelwidth*\i},every node/.append style={
                 yslant=0,xslant=1,rotate=\rot},yslant=0,xslant=1,rotate=\rot
              ]              
                \coordinate (EB\convid\i) at ($(\xx,\yy)+(0,0)$);
                \coordinate (DEB\convid\i) at ($(\xx,\yy)+(0.5*5/3,0.35*5/3)$);
              \gridd{EB\convid\i}{0.5*5/3}{0.35*5/3}
        \end{scope}
    }

    \draw[fill,gray!10,opacity=0.8] (EB\convid0) -- (EB\convid\nin) -- ($(EB\convid\nin)+(0.5*5/3,0)$) -- ($(EB\convid0)+(0.5*5/3,0)$) --cycle; 
    \draw[fill,gray!10,opacity=0.8] (DEB\convid0) -- (DEB\convid\nin) -- ($(EB\convid\nin)+(0.5*5/3,0)$) -- ($(EB\convid0)+(0.5*5/3,0)$) --cycle; 
    \foreach \i in {0,1,...,\nin} {
        \begin{scope}[  % Lower layer
            yshift={-100+\channelwidth*\i},every node/.append style={
                 yslant=0,xslant=1,rotate=\rot},yslant=0,xslant=1,rotate=\rot
              ]              
            \draw[gray!80] (EB\convid\i) -- ($(EB\convid\i)+(0.5*5/3,0)$);
            \draw[gray!80] (DEB\convid\i) -- ($(EB\convid\i)+(0.5*5/3,0)$);
        \end{scope}
    }
    \draw[black] (EB\convid0) -- (EB\convid\nin) -- ($(EB\convid\nin)+(0.5*5/3,0)$) -- ($(EB\convid0)+(0.5*5/3,0)$) --cycle; 
    \draw[black] (DEB\convid0) -- (DEB\convid\nin) -- ($(EB\convid\nin)+(0.5*5/3,0)$) -- ($(EB\convid0)+(0.5*5/3,0)$) --cycle; 
}

\newcommand\bconvvv[4]{
    \def\nin{#1}
    \def\convid{#2}
    \def\xx{#3}
    \def\yy{#4}
    \foreach \i in {0,1,...,\nin} {
        \begin{scope}[  % Lower layer
            yshift={-320+\channelwidth*\i},every node/.append style={
                 yslant=0,xslant=1,rotate=\rot},yslant=0,xslant=1,rotate=\rot
              ]              
                \coordinate (EB\convid\i) at ($(\xx,\yy)+(0,0)$);
                \coordinate (DEB\convid\i) at ($(\xx,\yy)+(0.5*7/3,0.35*7/3)$);
              \griddd{EB\convid\i}{0.5*7/3}{0.35*7/3}
        \end{scope}
    }

    \draw[fill,gray!10,opacity=0.8] (EB\convid0) -- (EB\convid\nin) -- ($(EB\convid\nin)+(0.5*7/3,0)$) -- ($(EB\convid0)+(0.5*7/3,0)$) --cycle; 
    \draw[fill,gray!10,opacity=0.8] (DEB\convid0) -- (DEB\convid\nin) -- ($(EB\convid\nin)+(0.5*7/3,0)$) -- ($(EB\convid0)+(0.5*7/3,0)$) --cycle; 
    \foreach \i in {0,1,...,\nin} {
        \begin{scope}[  % Lower layer
            yshift={-100+\channelwidth*\i},every node/.append style={
                 yslant=0,xslant=1,rotate=\rot},yslant=0,xslant=1,rotate=\rot
              ]              
            \draw[gray!80] (EB\convid\i) -- ($(EB\convid\i)+(0.5*7/3,0)$);
            \draw[gray!80] (DEB\convid\i) -- ($(EB\convid\i)+(0.5*7/3,0)$);
        \end{scope}
    }
    \draw[black] (EB\convid0) -- (EB\convid\nin) -- ($(EB\convid\nin)+(0.5*7/3,0)$) -- ($(EB\convid0)+(0.5*7/3,0)$) --cycle; 
    \draw[black] (DEB\convid0) -- (DEB\convid\nin) -- ($(EB\convid\nin)+(0.5*7/3,0)$) -- ($(EB\convid0)+(0.5*7/3,0)$) --cycle; 
}

\newcommand\relu[2]{
        \def\nodename{#1}
        \def\writename{#2}
        \draw[blue!40!black!80, thick] (\nodename) circle [radius=0.35];
        \draw[blue!40!black!80, thick] ($(\nodename) + (-0.35,0)$) -- (\nodename);
        \draw[blue!40!black!80, thick] ($(\nodename) + (0.35*1.414/2,0.35*1.414/2)$) -- (\nodename);
        \node[blue!40!black!80] at ($(\nodename)+(0,-0.9)$) {ReLU};
        \node[blue!40!black!80] at ($(\nodename)+(0,+0.7)$) {\writename};
}
\begin{figure}[t]
    \centering
\def\rot{0}
\def\dd{0.8}
\def\ggg{0.6}
\def\c{gray!10}
\def\d{green!20}
\def\channelwidth{8}
\def\fmw{1.0}
\def\fmh{0.7}
\resizebox{0.97\columnwidth}{!}{
\pgfdeclarelayer{fig}\pgfdeclarelayer{bg}
\pgfsetlayers{bg,fig}  
\begin{tikzpicture}[decoration={brace},font=\Large][scale=1,every node/.style={minimum size=1cm},on grid]

%%%%%%%%%%%% UPPER ROW %%%%%%%%%%%%%%%%%%%%

\begin{pgfonlayer}{fig}

        \gfm{2}{0}{-4.5}{0.0}{-20}
        \bbconv{2}{0}{-2.4}{0.15}
        \bbconv{2}{1}{-1.6}{0.15}
        % \bbconv{3}{2}{-0.8}{0.15}
        % \bbconv{4}{3}{0.0}{0.15}
        % \bbconv{4}{4}{0.8}{0.15}
        % \bbconv{4}{5}{1.6}{0.15}
        \gfm{2}{1}{0.4}{-0.0}{-20}
        \bbconv{2}{6}{2.1}{0.15}
        \bbconv{2}{7}{2.9}{0.15}
        % \bbconv{3}{8}{4.5}{0.15}
        % \bbconv{4}{9}{6.1}{0.15}
        \gfm{2}{2}{5.2}{0.0}{-20}
        % \bbconv{2}{8}{10.3}{0.15}
        % \bbconv{2}{9}{11.1}{0.15}
        % \bbconv{2}{8}{11.9}{0.15}
        % \bbconv{2}{9}{12.7}{0.15}
        % \bbconv{2}{10}{13.5}{0.15}
        % \bbconv{2}{11}{14.3}{0.15}
        
        % \ggfm{6}{3}{16.6}{0.0}{-20}
        % \bbconv{6}{11}{18.7}{0.15}
        % \bbconv{6}{12}{19.5}{0.15}
        % \bbconv{6}{13}{20.3}{0.15}
        % \gfm{3}{4}{22.3}{0.0}{-20}
        % \bbconv{3}{14}{24.0}{0.15}
        % \bbconv{3}{15}{24.8}{0.15}
        % \bbconv{3}{16}{25.6}{0.15}
        % \gfm{3}{5}{27.9}{0.0}{-20}

        \coordinate (id) at (-0.1,-0.2);
        \coordinate (id2) at (4.7,-0.2);
        \coordinate (id3) at (21.8,-0.2);
        \coordinate (act) at (16.05, -0.2);
        \coordinate (act2) at (27.4, -0.2);
        \draw[black, thick] (id) circle [radius=0.35];
        \draw[black, thick] ($(id) + (0.35*1.414/2,0.35*1.414/2)$) --($(id) + (-0.35*1.414/2,-0.35*1.414/2)$);
        \node[black] at ($(id)+(0,+0.75)$) {$\mathrm{id}$};
        % \draw[black, thick] (id2) circle [radius=0.35];
        % \draw[black, thick] ($(id2) + (0.35*1.414/2,0.35*1.414/2)$) --($(id2) + (-0.35*1.414/2,-0.35*1.414/2)$);
        % \node[black] at ($(id2)+(0,+0.75)$) {$\mathrm{id}$};
        % \draw[black, thick] (id3) circle [radius=0.35];
        % \draw[black, thick] ($(id3) + (0.35*1.414/2,0.35*1.414/2)$) --($(id3) + (-0.35*1.414/2,-0.35*1.414/2)$);
        % \node[black] at ($(id3)+(0,+0.75)$) {$\mathrm{id}$};

        \node[black!100] at ($(A02)+(2.7,1.4)$) {$\theta_1$};
        \node[black!100] at ($(A12)+(2.3,1.4)$) {$\theta_2$};

        \draw[blue!40!black!80, thick] (id2) circle [radius=0.35];
        \draw[blue!40!black!80, thick] ($(id2) + (-0.35,0)$) -- (id2);
        \draw[blue!40!black!80, thick] ($(id2) + (0.35*1.414/2,0.35*1.414/2)$) -- (id2);
        \node[blue!40!black!80] at ($(id2)+(0,-0.9)$) {ReLU};
        \node[blue!40!black!80] at ($(id2)+(0,+0.7)$) {$\sigma_2$};

        % \draw[blue!60!black!100, thick] (act) circle [radius=0.35];
        % \draw[blue!60!black!100, thick] ($(act) + (-0.35,0)$) -- (act);
        % \draw[blue!60!black!100, thick] ($(act) + (0.35*1.414/2,0.35*1.414/2)$) -- (act);
        % \node[blue!60!black!100] at ($(act)+(0,-0.9)$) {ReLU};
        % \node[blue!60!black!100] at ($(act)+(0,+0.7)$) {$\sigma_3$};

        % \draw[black, thick] (act2) circle [radius=0.35];
        % \draw[black, thick] ($(act2) + (0.35*1.414/2,0.35*1.414/2)$) --($(act2) + (-0.35*1.414/2,-0.35*1.414/2)$);
        % \node[black] at ($(act2)+(0,+0.75)$) {$\mathrm{id}$};
        \draw[rounded corners,line width=0.5mm,inner sep=20pt, black!100] ($(B70)+(1.0,-0.4)$) rectangle ($(B00)+(-0.2,1.4)$) {};
        % \draw[rounded corners,line width=0.5mm,inner sep=20pt, black!100] ($(B160)+(1.0,-0.4)$) rectangle ($(B110)+(-0.2,2.4)$) {};

    % \node[black!100] () at ($(B60)+(-0.2,-1.2)$) {\huge $\ast$};

    \node[black!100] at ($(A02)+(0.9,1.35)$) {$X^{(0)}$};
    \node[black!100] at ($(A12)+(0.9,1.35)$) {$X^{(1)}$};
    \node[black!100] at ($(A22)+(0.9,1.35)$) {$X^{(2)}$};
    % \node[red!60!black!100] at ($(A33)+(0.9,2.15)$) {$X^{(3)}$};
    % \node[black!100] at ($(A43)+(0.9,2.15)$) {$X^{(4)}$};
    % \node[black!100] at ($(A52)+(0.9,2.4)$) {$X^{(5)}$};
    \end{pgfonlayer}
\begin{pgfonlayer}{bg}
    \draw[green!80!black!40] ($(DA00)+(-1,0)$) -- ($(DA02)+(-1,0)$) ;
    \draw[green!80!black!40] ($(DA10)+(-1,0)$) -- ($(DA12)+(-1,0)$) ;
    \draw[green!80!black!40] ($(DA20)+(-1,0)$) -- ($(DA22)+(-1,0)$) ;
    % \draw[red!80!black!60] ($(DA30)+(-1,0)$) -- ($(DA34)+(-1,0)$) ;
    % \draw[green!80!black!40] ($(DA40)+(-1,0)$) -- ($(DA43)+(-1,0)$) ;
    % \draw[green!80!black!40] ($(DA50)+(-1,0)$) -- ($(DA53)+(-1,0)$) ;
    \end{pgfonlayer}

%%%%%% UPPER MERGED %%%%%%%%%%%%%%%%%%%%%%%%
\begin{pgfonlayer}{fig}
    \gfm{2}{0}{-4.5}{0.0}{-110}
    \bconvv{2}{1}{-2.4}{0.0}
    \bconvv{2}{2}{-0.8}{0.0}
    % \bconvv{3}{3}{0.8}{0.0}
    \gfm{2}{1}{2.2}{0.0}{-110}
        \coordinate (id) at (1.6,-3.52);

        \draw[blue!40!black!80, thick] (id) circle [radius=0.35];
        \draw[blue!40!black!80, thick] ($(id) + (-0.35,0)$) -- (id);
        \draw[blue!40!black!80, thick] ($(id) + (0.35*1.414/2,0.35*1.414/2)$) -- (id);
        \node[blue!40!black!80] at ($(id)+(0,-0.9)$) {ReLU};
        \node[blue!40!black!80] at ($(id)+(0,+0.7)$) {$\sigma_2$};

    \draw[rounded corners,line width=0.5mm,inner sep=20pt, black!100] ($(EB20)+(1.7,-0.3)$) rectangle ($(EB10)+(-0.2,1.4)$) {};
    % \draw[rounded corners,line width=0.5mm,inner sep=20pt, black!100] ($(EB100)+(1.6,-0.3)$) rectangle ($(EB80)+(-0.2,2.6)$) {};

    \node[black!100] () at ($(EB10)!0.5!(EB20) + (0.75,1.83)$) {$\theta_\mathrm{merged} = \theta_2 \ast \theta_1$};
    % \node () at ($(B100)!0.5!(B90) + (-0.1,-1.4)$) {$\theta_3\in\mathbb{R}^{2\times 6\times K\times K}$};
    % \node () at ($(B160)!0.5!(B110) + (-0.5,-1.4)$) {$\theta_5 \ast \theta_4\in\mathbb{R}^{6\times 3\times(2K-1)\times(2K-1)}$};
    % \draw[->, line width=2pt,black!100] ($(B00)+(4.4,-0.4)$) to [out=-90,in=130] ($(EB10)+(-0.3,2.0)$);
    % \draw[->, line width=2pt,black!100] ($(B110)+(-0.2,0.2)$) to [out=240,in=130] ($(EB80)+(-0.3,2.0)$);
    % \node[black!100] () at ($(EB80)+(-1.1,3.6)$) {\LARGE $\ast$};
    % \node[black!100] () at ($(EB10)+(-1.6,3.6)$) {\LARGE $\ast$};
    % \node[black!100] () at ($(B60)+(-5.1,2.8)$) {\LARGE $\ast$};

    % \node[black!100] at ($(A12)+(-2.3,-2.2)$) {$\mathrm{Ker}(\theta_1) = \mathrm{Ker}(\theta_2) = 3$, $\mathrm{Ker}(\theta_\mathrm{merged}) = 5$};%, Latency speed-up : 1.45$\times$};
    % % \node[black!100] at ($(A12)+(6.15,0.5)$) {$\theta_\mathrm{merged} = \theta_2 \ast \theta_1 \in\mathbb{R}^{\mathrm{C}_\mathrm{in}\times \mathrm{C}_\mathrm{out}\times 5\times 5}$};
    % \node[black!100] at ($(A12)+(-3.8,-3.2)$) {Latency speed-up : 1.45$\times$};

    \node[black!100] at ($(A12)+(5.07,0.8)$) {$\mathrm{Ker}(\theta_1) = \mathrm{Ker}(\theta_2) = 3$};
    \node[black!100] at ($(A12)+(4.4,-0.2)$) {$\mathrm{Ker}(\theta_\mathrm{merged}) = 5$};
    % \node[black!100] at ($(A12)+(6.15,0.5)$) {$\theta_\mathrm{merged} = \theta_2 \ast \theta_1 \in\mathbb{R}^{\mathrm{C}_\mathrm{in}\times \mathrm{C}_\mathrm{out}\times 5\times 5}$};
    \node[black!100] at ($(A12)+(5.43,-1.2)$) {$\mathrm{Latency~speed\text{-}up :}~1.45\times$};

    \node[black!100] at ($(A02)+(0.9,1.35)$) {$X^{(0)}$};
    \node[black!100] at ($(A12)+(0.9,1.35)$) {$X^{(2)}$};
    % \node[red!60!black!100] at ($(BA23)+(0.9,1.35)$) {$X^{(3)}$};
    % \node[black!100] at ($(A32)+(0.9,2.4)$) {$X^{(5)}$};
    \end{pgfonlayer} 
\begin{pgfonlayer}{bg}
    \draw[green!80!black!40] ($(DA00)+(-1,0)$) -- ($(DA02)+(-1,0)$) ;
    \draw[green!80!black!40] ($(DA10)+(-1,0)$) -- ($(DA12)+(-1,0)$) ;
    % \draw[red!80!black!60] ($(DBA20)+(-1,0)$) -- ($(DBA24)+(-1,0)$) ;
    % \draw[green!80!black!40] ($(DA30)+(-1,0)$) -- ($(DA33)+(-1,0)$) ;
    % \node[align=left] () at (-2,-3.5) {\textcolor{blue!60!black!100}{$A=\{3\}$}\\\textcolor{red!60!black!100}{$S=\{2,3\}$}};
    \draw[->, line width=2pt,black!100] ($(B00)+(3.9,-0.4)$) to [out=-90,in=50] ($(EB10)+(3.4,1.5)$);
    \end{pgfonlayer}

%%%%%%%%%%%% LOWER ROW %%%%%%%%%%%%%%%%%%%%

\begin{pgfonlayer}{fig}

        \gfm{2}{0}{-4.5}{0.0}{-230}
        \bbbconv{2}{0}{-2.4}{0.15}
        \bbbconv{2}{1}{-1.6}{0.15}
        % \bbbconv{2}{2}{-0.8}{0.15}
        \gfm{2}{1}{0.4}{-0.0}{-230}
        \bbbconv{2}{6}{2.1}{0.15}
        \bbbconv{2}{7}{2.9}{0.15}
        % \bbbconv{3}{8}{4.5}{0.15}
        \gfm{2}{2}{4.9}{0.0}{-230}
        \bbbconv{2}{12}{6.6}{0.15}
        \bbbconv{2}{13}{7.4}{0.15}
        % \bbbconv{3}{14}{9.8}{0.15}
        \gfm{2}{3}{9.7}{0.0}{-230}
        % \bbbconv{3}{18}{13.5}{0.15}
        % \bbbconv{3}{19}{14.3}{0.15}
        % \bbbconv{3}{20}{15.1}{0.15}
        % \gfm{3}{4}{17.4}{0.0}{-230}
        
        % \bbconv{2}{8}{10.3}{0.15}
        % \bbconv{2}{9}{11.1}{0.15}
        % \bbconv{2}{8}{11.9}{0.15}
        % \bbconv{2}{9}{12.7}{0.15}
        % \bbconv{2}{10}{13.5}{0.15}
        % \bbconv{2}{11}{14.3}{0.15}
        
        % \ggfm{6}{3}{16.6}{0.0}{-20}
        % \bbconv{6}{11}{18.7}{0.15}
        % \bbconv{6}{12}{19.5}{0.15}
        % \bbconv{6}{13}{20.3}{0.15}
        % \gfm{3}{4}{22.3}{0.0}{-20}
        % \bbconv{3}{14}{24.0}{0.15}
        % \bbconv{3}{15}{24.8}{0.15}
        % \bbconv{3}{16}{25.6}{0.15}
        % \gfm{3}{5}{27.9}{0.0}{-20}

        \coordinate (id) at (-0.1,-7.7);
        \coordinate (id2) at (4.4,-7.7);
        % \coordinate (id3) at (11.3,-8.3);
        \coordinate (act) at (9.2, -7.7);
        \coordinate (act2) at (27.4, -9.1);
        \draw[black, thick] (id) circle [radius=0.35];
        \draw[black, thick] ($(id) + (0.35*1.414/2,0.35*1.414/2)$) --($(id) + (-0.35*1.414/2,-0.35*1.414/2)$);
        \node[black] at ($(id)+(0,+0.75)$) {$\mathrm{id}$};
        % \draw[black, thick] (id2) circle [radius=0.35];
        % \draw[black, thick] ($(id2) + (0.35*1.414/2,0.35*1.414/2)$) --($(id2) + (-0.35*1.414/2,-0.35*1.414/2)$);
        % \node[black] at ($(id2)+(0,+0.75)$) {$\mathrm{id}$};
        % \draw[black, thick] (id3) circle [radius=0.35];
        % \draw[black, thick] ($(id3) + (0.35*1.414/2,0.35*1.414/2)$) --($(id3) + (-0.35*1.414/2,-0.35*1.414/2)$);
        % \node[black] at ($(id3)+(0,+0.75)$) {$\mathrm{id}$};

        % \draw[blue!40!black!80, thick] (id2) circle [radius=0.35];
        % \draw[blue!40!black!80, thick] ($(id2) + (-0.35,0)$) -- (id2);
        % \draw[blue!40!black!80, thick] ($(id2) + (0.35*1.414/2,0.35*1.414/2)$) -- (id2);
        % \node[blue!40!black!80] at ($(id2)+(0,-0.9)$) {ReLU};
        % \node[blue!40!black!80] at ($(id2)+(0,+0.7)$) {$\sigma_2$};
        
        \draw[black, thick] (id2) circle [radius=0.35];
        \node[black] at ($(id2)+(0,+0.75)$) {$\mathrm{id}$}; 
        \draw[black, thick] ($(id2) + (0.35*1.414/2,0.35*1.414/2)$) --($(id2) + (-0.35*1.414/2,-0.35*1.414/2)$);
        
        % \draw[black, thick] (id3) circle [radius=0.35];
        % \node[black] at ($(id3)+(0,+0.75)$) {$\mathrm{id}$}; 
        % \draw[black, thick] ($(id3) + (0.35*1.414/2,0.35*1.414/2)$) --($(id3) + (-0.35*1.414/2,-0.35*1.414/2)$);

        \draw[blue!60!black!100, thick] (act) circle [radius=0.35];
        \draw[blue!60!black!100, thick] ($(act) + (-0.35,0)$) -- (act);
        \draw[blue!60!black!100, thick] ($(act) + (0.35*1.414/2,0.35*1.414/2)$) -- (act);
        \node[blue!60!black!100] at ($(act)+(0,-0.9)$) {ReLU};
        \node[blue!60!black!100] at ($(act)+(0,+0.7)$) {$\sigma_3$};

        % \draw[black, thick] (act2) circle [radius=0.35];
        % \draw[black, thick] ($(act2) + (0.35*1.414/2,0.35*1.414/2)$) --($(act2) + (-0.35*1.414/2,-0.35*1.414/2)$);
        % \node[black] at ($(act2)+(0,+0.75)$) {$\mathrm{id}$};
        \draw[rounded corners,line width=0.5mm,inner sep=20pt, black!100] ($(B70)+(5.5,-0.4)$) rectangle ($(B00)+(-0.2,1.4)$) {};
        % \draw[rounded corners,line width=0.5mm,inner sep=20pt, black!100] ($(B160)+(1.0,-0.4)$) rectangle ($(B110)+(-0.2,2.4)$) {};

    % \node[black!100] () at ($(B60)+(0.6,-1.5)$) {\huge $\ast$};

    \node[black!100] at ($(A02)+(2.7,1.4)$) { $\theta_1$};
    \node[black!100] at ($(A12)+(2.3,1.4)$) { $\theta_2$};
    \node[black!100] at ($(A22)+(2.3,1.4)$) { $\theta_3$};

    \node[black!100] at ($(A02)+(0.9,1.35)$) {$X^{(0)}$};
    \node[black!100] at ($(A12)+(0.9,1.35)$) {$X^{(1)}$};
    \node[black!100] at ($(A22)+(0.9,1.35)$) {$X^{(2)}$};
    \node[black!100] at ($(A32)+(0.9,1.35)$) {$X^{(3)}$};
    % \node[black!100] at ($(A43)+(0.9,1.35)$) {$X^{(4)}$};
    % \node[black!100] at ($(A52)+(0.9,2.4)$) {$X^{(5)}$};
    \end{pgfonlayer}
\begin{pgfonlayer}{bg}
    \draw[green!80!black!40] ($(DA00)+(-1,0)$) -- ($(DA02)+(-1,0)$) ;
    \draw[green!80!black!40] ($(DA10)+(-1,0)$) -- ($(DA12)+(-1,0)$) ;
    \draw[green!80!black!40] ($(DA20)+(-1,0)$) -- ($(DA22)+(-1,0)$) ;
    % \draw[red!80!black!60] ($(DA30)+(-1,0)$) -- ($(DA34)+(-1,0)$) ;
    % \draw[green!80!black!40] ($(DA40)+(-1,0)$) -- ($(DA43)+(-1,0)$) ;
    % \draw[green!80!black!40] ($(DA50)+(-1,0)$) -- ($(DA53)+(-1,0)$) ;
    \end{pgfonlayer}

%%%%%% LOWER MERGED %%%%%%%%%%%%%%%%%%%%%%%%

\begin{pgfonlayer}{fig}
    \gfm{2}{0}{-4.5}{0.0}{-320}
    \bconvvv{2}{1}{-2.4}{0.0}
    \bconvvv{2}{2}{-0.8}{0.0}
    % \bconvvv{3}{3}{0.8}{0.0}
    \gfm{2}{1}{2.4}{-0.0}{-320}
        % \coordinate (act) at (16.05, -4.42);
        % \coordinate (act2) at (24.77, -4.42);
        \coordinate (id) at (1.9,-11.02);

        \draw[blue!40!black!80, thick] (id) circle [radius=0.35];
        \draw[blue!40!black!80, thick] ($(id) + (-0.35,0)$) -- (id);
        \draw[blue!40!black!80, thick] ($(id) + (0.35*1.414/2,0.35*1.414/2)$) -- (id);
        \node[blue!40!black!80] at ($(id)+(0,-0.9)$) {ReLU};
        \node[blue!40!black!80] at ($(id)+(0,+0.7)$) {$\sigma_3$};

    \draw[rounded corners,line width=0.5mm,inner sep=20pt, black!100] ($(EB20)+(2.2,-0.3)$) rectangle ($(EB10)+(-0.2,1.6)$) {};
    % \draw[rounded corners,line width=0.5mm,inner sep=20pt, black!100] ($(EB100)+(1.6,-0.3)$) rectangle ($(EB80)+(-0.2,2.6)$) {};

    \node[black!100] () at ($(EB10)!0.5!(EB20) + (1.25, 1.93)$) {$\theta_\mathrm{merged} = \theta_3 \ast \theta_2 \ast \theta_1$};
    % \node () at ($(B100)!0.5!(B90) + (-0.1,-1.4)$) {$\theta_3\in\mathbb{R}^{2\times 6\times K\times K}$};
    % \node () at ($(B160)!0.5!(B110) + (-0.5,-1.4)$) {$\theta_5 \ast \theta_4\in\mathbb{R}^{6\times 3\times(2K-1)\times(2K-1)}$};
    % \draw[->, line width=2pt,black!100] ($(B00)+(4.4,-0.4)$) to [out=-90,in=130] ($(EB10)+(-0.3,2.0)$);
    % \draw[->, line width=2pt,black!100] ($(B110)+(-0.2,0.2)$) to [out=240,in=130] ($(EB80)+(-0.3,2.0)$);
    % \node[black!100] () at ($(EB80)+(-1.1,3.6)$) {\LARGE $\ast$};
    % \node[black!100] () at ($(EB10)+(-1.6,3.6)$) {\LARGE $\ast$};
    % \node[black!100] () at ($(B60)+(-5.1,2.8)$) {\LARGE $\ast$};

    \node[black!100] at ($(A02)+(0.9,1.35)$) {$X^{(0)}$};
    \node[black!100] at ($(A12)+(0.9,1.35)$) {$X^{(2)}$};

    \node[black!100] at ($(A12)+(6.2,0.8)$) {$\mathrm{Ker}(\theta_1) = \mathrm{Ker}(\theta_2) = \mathrm{Ker}(\theta_3) = 3$};
    \node[black!100] at ($(A12)+(4.4,-0.2)$) {$\mathrm{Ker}(\theta_\mathrm{merged}) = 7$};
    % \node[black!100] at ($(A12)+(6.15,0.5)$) {$\theta_\mathrm{merged} = \theta_2 \ast \theta_1 \in\mathbb{R}^{\mathrm{C}_\mathrm{in}\times \mathrm{C}_\mathrm{out}\times 5\times 5}$};
    \node[black!100] at ($(A12)+(5.43,-1.2)$) {$\mathrm{Latency~speed\text{-}up :}~1.02\times$};

    % \node[black!100] at ($(A12)+(-0.2,-2.2)$) {$\theta_1, \theta_2, \theta_3 \in\mathbb{R}^{\mathrm{C}_\mathrm{in}\times \mathrm{C}_\mathrm{out}\times 3\times 3},~\theta_\mathrm{merged} = \theta_3 \ast \theta_2 \ast \theta_1 \in\mathbb{R}^{\mathrm{C}_\mathrm{in}\times \mathrm{C}_\mathrm{out}\times 7\times 7}$};
    % % \node[black!100] at ($(A12)+(6.15,0.5)$) {$\theta_\mathrm{merged} = \theta_2 \ast \theta_1 \in\mathbb{R}^{\mathrm{C}_\mathrm{in}\times \mathrm{C}_\mathrm{out}\times 5\times 5}$};
    % \node[black!100] at ($(A12)+(-4.0,-3.2)$) {Latency speed-up : 1.02$\times$};
    % % \node[red!60!black!100] at ($(BA23)+(0.9,1.35)$) {$X^{(3)}$};
    % % \node[black!100] at ($(A32)+(0.9,2.4)$) {$X^{(5)}$};
    \end{pgfonlayer} 
\begin{pgfonlayer}{bg}
    \draw[green!80!black!40] ($(DA00)+(-1,0)$) -- ($(DA02)+(-1,0)$) ;
    \draw[green!80!black!40] ($(DA10)+(-1,0)$) -- ($(DA12)+(-1,0)$) ;
    % \draw[red!80!black!60] ($(DBA20)+(-1,0)$) -- ($(DBA24)+(-1,0)$) ;
    % \draw[green!80!black!40] ($(DA30)+(-1,0)$) -- ($(DA33)+(-1,0)$) ;
    % \node[align=left] () at (-2,-3.5) {\textcolor{blue!60!black!100}{$A=\{3\}$}\\\textcolor{red!60!black!100}{$S=\{2,3\}$}};
    \draw[->, line width=2pt,black!100] ($(B00)+(4.8,-0.4)$) to [out=-90,in=20] ($(EB10)+(3.8,1.4)$);
    \end{pgfonlayer}

\end{tikzpicture}
}

\caption{
\looseness=-1 An illustration of the increase in kernel size significantly undermining the latency reduction in the depth compression framework.
Here, $\theta_l$ denotes the $l$-th convolution parameter, $X^{(l)}$ denotes the $l$-th feature map, and $\mathrm{Ker}(\cdot)$ denotes the kernel size of the parameter.
As the layers are merged, the kernel size of the merged layer continues to grow, impeding the latency reduction.
The latency is measured for the depicted model, on RTX2080 Ti, with channel size 256, input resolution $56 \times 56$, and batch size 128.
}
\label{fig:merge}
\vspace{-0.5em}
\end{figure}
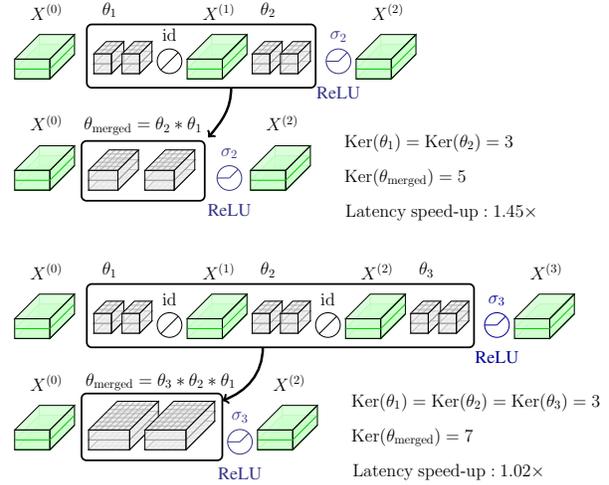

\begin{figure*}[t]
    \centering
    \includegraphics[width=1.0\linewidth]{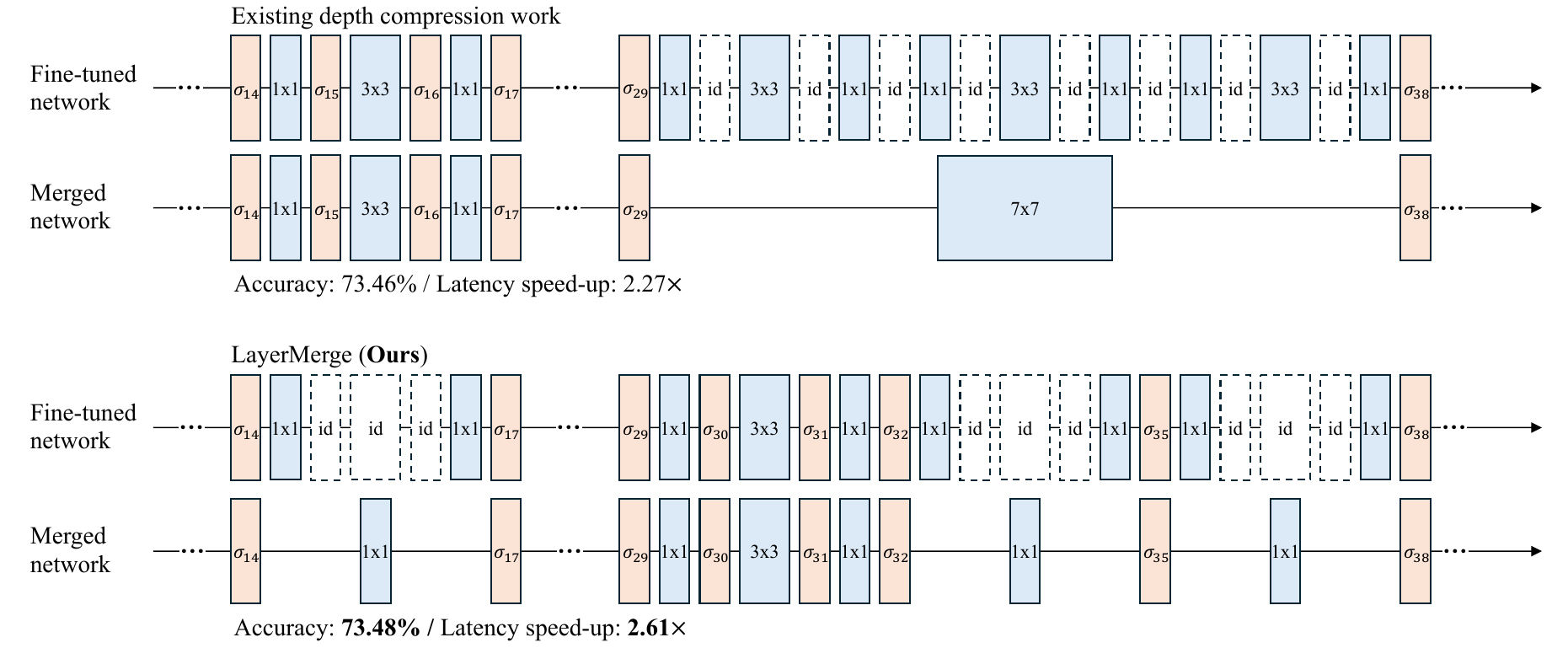}
    \caption{
    A qualitative example comparing our method to the depth compression baseline \citep{kim23efficient}, applied to MobileNetV2-1.4 model on ImageNet dataset.
    Existing depth compression methods have limitations in reducing latency due to the inevitable increase in the kernel size of the merged layer.
    Our method effectively bypasses this challenge by jointly optimizing the selection of the convolution layers and the non-linear activation layers.
    %Both examples are taken from the compression results applying each method to MobileNetV2-1.4 architecture on ImageNet dataset. 
    }
    \label{fig:qual}
\end{figure*}

Convolutional neural networks (CNNs) have shown remarkable performance in various vision-based tasks such as classification, segmentation, and object detection \citep{alexnet, deeplabv3plus, fastrcnn}.
More recently, diffusion probabilistic models employing U-Net based architecture are demonstrating great performance in various high-quality image generation tasks \citep{ddpm, unet}.
However, the impressive capabilities of these models on complex vision tasks come at the cost of increasingly higher computational resources and inference latency as they are scaled up \citep{iddpm,convnext}.

An effective approach to address this is \emph{structured pruning}, which consists of removing redundant regular regions of weights, such as channels, filters, and entire layers, to make the model more efficient without requiring specialized hardware while preserving its performance.
In particular, \textit{channel pruning} methods remove redundant channels in CNNs \citep{taylor16,pfec,halp}.
While these methods have shown significant improvements in accelerating models and reducing their complexity, they are less effective on architectures with a low number of channels compared to methods that reduce the number of layers in the network \citep{layerprune,depthshrinker,kim23efficient}.
One approach to reduce the number of layers is \textit{layer pruning}, which removes entire layers in the network \citep{shallowing,layerprune}.
Although such methods achieve a larger speed-up factor, they tend to suffer from severe degradation in performance due to their aggressive nature in removing parameters \citep{depthshrinker}.

To this end, a line of research called \textit{depth compression} or \textit{depth shrinking} proposes to remove redundant non-linear activation functions and to merge the resulting consecutive convolution layers into a single layer to achieve higher inference speed-up \citep{layerfolding,depthshrinker,kim23efficient}.
However, these methods have a fundamental limitation; merging consecutive convolution layers leads to an increase in the kernel size of the merged layers, significantly hindering the latency reduction gained from reducing the depth of the network (\Cref{fig:merge}).

In this work, we argue that this critical drawback of existing depth compression methods can be addressed by extending the search space and jointly optimizing the selection of both \textit{non-linear activation layers} and the \textit{convolution layers} to remove (\Cref{fig:qual}).
Since the corresponding optimization problem involves an exponential number of potential solutions, we propose a novel surrogate optimization problem that can be solved exactly via an efficient dynamic programming (DP) algorithm.

In particular, we make the following contributions:
\begin{itemize}\itemsep=0pt
    \item We propose a novel depth compression approach that introduces a new pruning modality; removing the convolution layers in addition to activation layers.
    This formulation encompasses the search space of existing depth compression methods and 
    enables us to bypass the increase in the kernel size of the merged convolution layers (\Cref{fig:qual}).
    \item We propose a surrogate optimization problem that can be solved exactly via a DP algorithm.
    We further develop an efficient method to construct DP lookup tables, leveraging the inherent structure of the problem.
    \item We conduct extensive experiments and demonstrate the effectiveness of our method on image classification and generation tasks with different networks, including ResNet, MobileNetV2, and DDPM \citep{resnet,mobilenetv2,ddpm}.
\end{itemize}

\section{Preliminaries}

Let $f_{\theta_l}$ and $\sigma_l$ denote the $l$-th convolution layer and the $l$-th activation layer, respectively.
Here, $\theta_l$ denotes the parameters of the convolution layer $f_{\theta_l}$. 
An $L$-layer CNN can be represented as $\bigO_{l=1}^L (\sigma_l \circ f_{\theta_l})$, where $\bigO$ denotes an iterated function composition, and the last activation function $\sigma_L$ is set to identity.

Depth compression methods eliminate less important non-linear activation layers, then \emph{merge} consecutive convolution layers into a single layer by applying a convolution operation to their parameters \citep{layerfolding,depthshrinker,kim23efficient}.
This approach leverages the fact that the successive convolution layers $f_{\theta_j} \circ f_{\theta_{j-1}} \circ \cdots \circ f_{\theta_i}$ can be represented as single equivalent convolution layer $f_{\theta_j \ast \cdots \ast \theta_i}$, where $\ast$ denotes the convolution operation.

However, this method has a fundamental limitation; the kernel size of the merged layer increases as more layers are merged.
This significantly undermines the latency speed-up achieved by reducing the number of layers (\Cref{fig:merge}).
To illustrate, let us denote the convolution parameter of the merged layer as $\hat{\theta} \coloneqq \theta_j \ast \cdots \ast \theta_i$ and assume that all convolution layers have a stride of $1$.
Then, the kernel size of this merged layer is given by
\begin{equation}
    \label{eq:kernel}
    \mathrm{Ker}(\hat{\theta}) = 1 + \sum_{l=i}^j \left( \mathrm{Ker}\left(\theta_l\right)-1\right), 
\end {equation}
where the $\mathrm{Ker}(\cdot)$ denotes the kernel size of the given convolution parameter.
To address this, we propose to jointly remove unimportant convolution layers and non-linear activation layers (\Cref{fig:qual}).
%expand the search space of the depth compression work and include the options to remove convolution layers and non-linear activation layers in a jointly optimal fashion 

\section{LayerMerge}
In this section, we present our proposed depth compression method \emph{LayerMerge}, designed to address the increase in kernel size resulting from depth compression and find more efficient networks.
We first formulate the NP-hard subset selection problem we aim to solve.
Then, we present a surrogate optimization problem that can be exactly solved with an efficient dynamic programming algorithm.
Afterwards, we examine the theoretical optimality, complexity, and practical cost of our approach.

\subsection{Selection problem}
We observe that if we selectively replace certain convolution layers with identity functions ($\mathrm{id}$), we can effectively alleviate the problem of increasing kernel sizes resulting from merging layers.
Indeed, an identity layer can be represented with a $1 \times 1$ depthwise convolution layer, where the parameter values are set to $1$.
We denote the corresponding convolution parameters as $\theta_{\mathrm{id}} \in \mathbb{R}^{1 \times \mathrm{C}_{\mathrm{in}} \times 1 \times 1}$, where $\mathrm{C}_{\mathrm{in}}$ denotes the number of input channels.
Then, it is evident from \Cref{eq:kernel} that $\theta_{\mathrm{id}}$ does not contribute to the expansion of the kernel size.

To this end, we propose to optimize two sets of layer indices: $A \subseteq [L-1] \coloneqq \{1, \ldots, L-1\}$ for activation layers and $C \subseteq [L] = \{1, \ldots, L\}$ for convolution layers, where $L$ represents the depth of the original network. 
The indices in $A$ denote where we keep the original activation layers, while the indices in $C$ correspond to where the original convolution layers are maintained. 
For any layer $l$ not in $A$, we replace its activation function $\sigma_l$ with the $\mathrm{id}$ function. 
Similarly, if a layer $l$ is not in $C$, we substitute its convolution layer $f_{\theta_l}$ with the identity layer $f_{\theta_{\mathrm{id}}}$.

It is worth noting that it is non-trivial to remove a convolution layer when the shapes of the input and output feature maps are different.
To address this, we define a set of irreducible convolution layer indices $R$, where the input shape and the output shapes differ.
Concretely, we define $R$ as $R \coloneqq \{l \in [L] : \mathrm{Shape}(X^{(l-1)}) \neq  \mathrm{Shape}(X^{(l)})\}$, where $\mathrm{Shape}(\cdot)$ denotes the shape of a tensor and $X^{(l)}$ denotes the $l$-th layer feature map.
We then restrict the choice of $C$ to supersets of $R$, i.e.,  $R\subseteq C$.

Given a desired latency target $T_0 > 0$, our goal is to select $A$ and $C$ that maximize the performance of the resulting model after fine-tuning, while satisfying the latency target after merging.
We formulate this problem as follows:

{\small
\begin{align}\label{eq:master}
&\max_{A \subseteq [L-1], C\subseteq [L]}  \max_{\theta}  
\;\mathrm{Perf}\left( \bigO_{l=1}^L 
 \bigl( \sigma_{A, l}   \circ  f_{C, \theta_l, l} \bigr) \right) \\
&\text{subject to} \nonumber \\
&R \subseteq C, &   \hidewidth \text{(irreducible conv)} \nonumber\\
&\sigma_{A, l} = \left( \mathds{1}_A(l)\sigma_l + \left(1-\mathds{1}_A(l)\right) \mathrm{id} \right),  &  \hidewidth \text{(replaced act)}\nonumber\\
&f_{C, \theta_l, l} = \left(\mathds{1}_C(l) f_{\theta_{l}} + \left(1-\mathds{1}_C(l)\right)f_{\theta_\mathrm{id}} \right), &  \hidewidth \text{(replaced conv)}\nonumber\\
&\forall i \in [|A|+1]:\hat{\theta}_i= \Conv_{l=a_{i\!-\!1}\!+\!1}^{a_i} \left(\mathds{1}_C(l) \theta_{l}+ \left(1-\mathds{1}_C(l)\right)\theta_{\mathrm{id}} \right), \nonumber\\
& &  \hidewidth \text{(merged parameters)}\nonumber\\
&T\left( \bigO_{i=1}^{|A|} \bigl( \sigma_{a_i} \circ f_{\hat{\theta}_i}\bigr) \right) < T_0, &  \hidewidth\text{(latency constraint)}\nonumber
\end{align}
}
where $\conv$ denotes an iterated convolution operation, $a_0 = 0$, $a_{|A|+1} = L$, and $(a_i)_{i=1}^{|A|}$ denotes the elements of the set $A$ in ascending order.
Here, $\mathrm{Perf}(\cdot)$ and $T(\cdot)$ denote the performance and latency of the network, respectively.
The performance of the network is defined as a task-dependent metric: accuracy for classification tasks and negative diffusion loss for generation tasks \citep{ddpm}.
The indicator function $\mathds{1}_X (x)$ is equal to $1$ if $x \in X$, and $0$ otherwise.
We denote by $\sigma_{A, l}$ the $l$-th activation layer replaced according to set $A$, and by $f_{C, \theta_l, l}$ the $l$-th convolution layer replaced according to set $C$.
The parameter $\hat{\theta}_i$ is the $i$-th convolution layer in the merged network. 

Note that we used the pruned network before merging in the objective, while the latency constraint is applied to the merged network. Both networks represent the same function and yield the same output, and thus have the same performance objective. However, in practice, we observe that it is better to merge consecutive convolution layers only at inference time after fine-tuning is finished. We chose to use the network before merging in the objective to stress this.
% Note that we only merge  after fine-tuning is finished.
% Concretely, in Problem \eqref{eq:master}, we use the network before merging in the objective, while the latency constraint is applied to the merged network.

\subsection{Surrogate optimization problem}

Solving Problem \eqref{eq:master} in general is NP-hard. 
We propose to assign an importance value for each merged layer and optimize the sum of the importance values as a surrogate objective.
This is a common approximation used in the literature \cite{halp, spdy, kim23efficient}. We also approximate the overall latency of the merged network with the sum of the layer-wise latencies \citep{proxylessnas,halp}.
The main challenge that we face then is the exponentially large number of potential combinations of the merged layers that arise from the joint optimization over $C$.
To this end, we develop an efficient method for measuring latency and importance values, leveraging the inherent combinatorial structure of the problem.
Subsequently, we compute the optimal solutions of the surrogate problem in polynomial-time using a dynamic programming algorithm. %ensuring polynomial-time complexity.

\paragraph{Latency cost} 

We construct a latency lookup table for all possible merged layers.
A straightforward approach is to construct a table with entries $T[i, j, C]$ for all $i, j \in \{0, \ldots, L\}$, $i<j$, where each entry denotes the latency of the layer obtained by merging from the $(i+1)$-th layer to the $j$-th layer after replacing the convolution layers according to $C$.
However, this approach is not feasible because it requires measuring the latency for the exponential number of possible sets $C \cap (i,j]$.

To address this, we note that the choice of $C$ only affects the latency of a merged layer via the size of its kernel, since the number of input and output channels is fixed.
To this end, we propose to construct the latency table with entries $T[i, j, k]$, where the last index $k$ denotes the kernel size of the merged layer, given by $k = 1 + \sum_{l\in C \cap(i,j]} (\mathrm{Ker}(\theta_l) - 1)$.

Let $K_{ij}$ be the set of possible merged kernel sizes that can appear after merging from the $(i+1)$-th layer to the $j$-th layer.
Note that $|K_{ij}| \leq 1 + \sum_{l = i+1}^j (\mathrm{Ker}(\theta_l) - 1)$. 
Therefore, constructing the latency table with entries $T[i, j, k]$ requires $O(L^2 K_0)$ latency measurements, where $K_0 \coloneqq \sum_l \mathrm{Ker}(\theta_l)$ denotes the sum of the kernel sizes in the original network.
This is significantly lower than the $O(L^2 2^L)$ measurements needed to construct the latency table with entries $T[i,j,C]$.

\paragraph{Importance value}
%We further assign the importance value for each merged layer and propose to optimize the sum of the layer-wise importance as a surrogate objective.
Similarly to the latency cost case, we construct the importance lookup table with entries $I[i, j, k]$. 
Each entry denotes the importance of the layer obtained by merging from the $(i+1)$-th layer to the $j$-th layer, with the merged layer having a kernel size $k$.

We define the importance value of each merged layer as the change in the performance after replacing the corresponding part of the original network with the merged layer. 
However, multiple choices of $C$ can yield the same kernel size $k$ in the merged layer, but will vary in performance.

We propose to keep the convolution layers with the largest parameters $\ell_1$-norm among those resulting in the same merged kernel size $k$. 
This simple yet effective criterion is often used for channel and layer pruning \citep{pfec,layerprune}.
Concretely, for any $i, j \in \{0, \ldots, L\}$, $i<j$, and $k \in K_{ij}$, we let
\begin{align}\label{eq:ACij}
\widehat{C}_{ijk} \coloneqq & \argmax_{ C_{ij} \subseteq (i, j]} ~ \sum_{l \in C_{ij}} \| \theta_l\|_1  \\
&\text{subject to } ~ 1 + \!\!\sum_{l\in C_{ij}} (\mathrm{Ker}(\theta_l) - 1) = k, \nonumber\\
& \quad \quad \quad \quad ~~ R \cap (i, j] \subseteq C_{ij}.  \nonumber\\
\widetilde{C}_{ijk} \coloneqq &\{1, \ldots, i\} \cup \widehat{C}_{ijk} \cup \{j+1, \ldots, L\} \nonumber. \\
\widetilde{A}_{ij} \coloneqq &\{1, \ldots, i\} \cup \{j, \ldots, L-1\}.\nonumber 
\end{align}
% \begin{align}\label{eq:ACij}
% \widehat{C}_{ijk}:=~~\!\underset{ C_{ij} \subseteq (i, j]}{\mathrm{argmax}}\;~~~\sum_{l \in C_{ij}} \| \theta_l\|_1  \\
% &\mathrm{subject~to~} ~~ 1 + \!\!\sum_{l\in C_{ij}} (\mathrm{Ker}(\theta_l) - 1) = k, \nonumber\\
% & \quad \quad ~~~~ R \cap (i, j] \subseteq C_{ij}  \nonumber\\
% \widetilde{C}_{ijk} \coloneqq~~ &\{1, \ldots, i\} \cup \widehat{C}_{ijk} \cup \{j+1, \ldots, L\} \nonumber, \\
% \widetilde{A}_{ij} \coloneqq ~~&\{1, \ldots, i\} \cup \{j, \ldots, L-1\}.\nonumber 
% \end{align}

Computing $\widetilde{C}_{ijk}$ has a negligible cost. %requiring no additional fine-tuning.
We can now define the importance $I[i, j, k]$ as follows:

\begin{align}
\label{eq:imp}
    I[i, j, k] \coloneqq \exp & \left(
    \max_{\theta} \mathrm{Perf}\biggl(\bigO_{l=1}^L \bigl(\underbrace{\sigma_{\widetilde{A}_{ij}, l}}_{\text{Replaced act}} \circ \underbrace{f_{\widetilde{C}_{ijk}, \theta_l, l}}_{\text{Replaced conv}} \bigr)\biggr) \right. \nonumber \\
    & - \left. \max_{\theta} \mathrm{Perf} \biggl(\underbrace{\bigO_{l=1}^L \bigl(\sigma_{l} \circ f_{\theta_l} \bigr)}_{\text{Original network}}\biggr)\right).
\end{align}

We use $\exp(\cdot)$ to normalize the importance value.
This choice is based on our empirical observation that using positive values for importance leads to better performance by favoring solutions with more activation layers.
In practice, to estimate the first term, we measure the performance of the network after fine-tuning it for a few steps.
For the second term, we use the performance of the pre-trained original network.
Constructing the importance table requires $O(L^2 K_0)$ importance value evaluations, which is identical to the number of latency measurements needed.

\paragraph{Optimization problem}
After we pre-compute the latency and importance lookup tables, $T$ and $I$, %$T[i, j, k]$ and $I[i, j, k]$, 
we maximize the sum of the importance values of merged layer under the constraint on the sum of the latency costs.
This can be formulated as follows:

{\small
\begin{align} \label{eq:dp_master} %k_i \in K_{0 L}
&\max_{A \subseteq [L-1], k_i} \sum_{i=1}^{|A|+1} I[a_{i-1}, a_{i}, k_{i}] \\
&\text{subject to } ~~\sum_{i=1}^{|A|+1} T[a_{i-1}, a_i,k_i] < T_0, &\text{(latency constraint)} \nonumber\\
&~~~~~~~~~~~~~~~~~~~~k_i \in K_{a_{i-1} a_i},
&\text{(merged kernel size)}\nonumber
\end{align}
}
% {\small
% \begin{align} \label{eq:dp_master} %k_i \in K_{0 L}
% &\underset{A \subseteq [L-1], k_i}{\mathrm{maximize}} ~\;\sum_{i=1}^{|A|+1} I[a_{i-1}, a_{i}, k_{i}] \\
% &\mathrm{subject\ to\ } ~\sum_{i=1}^{|A|+1} T[a_{i-1}, a_i,k_i] < T_0, &\text{(latency constraint)} \nonumber\\
% &~~~~~~~~~~~~~~~~~~~~~k_i \in K_{a_{i-1} a_i},
% &\text{(merged kernel size)}\nonumber
% \end{align}
% }

where $a_0 = 0$, $a_{|A|+1} = L$ as before.
Given a solution $A^*, (k^*_i)_{i=1}^{|A^*|+1}$ of Problem \eqref{eq:dp_master}, the corresponding set of convolution layers we keep is given by $C^* = \bigcup_i \widehat{C}_{a^*_{i-1}, a^*_i, k^*_i}$.

\begin{figure}[t]
  \centering
    \begin{algorithm}[H]
    \caption{DP algorithm for Problem \eqref{eq:dp_master}}
    \label{alg:dp}
    \begin{algorithmic}
    \INPUT Importance $I$, latency $T$, latency budget $T_0$, discretization level $P$
    \STATE Initialize $M[0, t] \leftarrow 0$ for $t \geq 0$, $M[l, t] \leftarrow -\infty$ for $t < 0$, $A[0, t] \leftarrow \emptyset$, $C[0, t] \leftarrow \emptyset$
    \STATE Discretize latency values in $T$
    
    \FOR{$l=1$ {\bfseries to} $L$} 
    \FOR{$t \in \{\frac{T_0}{P}, \frac{2T_0}{P} \ldots, T_0\}$}
%    \FOR{$p=1$ {\bfseries to} $P$} 
%    \STATE $t \leftarrow \frac{pT_0}{P} $ 
%    \STATE $M[l, t] \leftarrow \underset{0 \leq l' < l,\,k }{\mathrm{max}}\!\! \left(M[l', t - T[l', l, k]]+ I[l', l, k]\right)\nonumber$
    \STATE $l^*, k^* \leftarrow \!\!\!\! \underset{0 \leq l' < l,\,k \in K_{l' l}}{\mathrm{argmax}}\!\!\!\!\left(M[l', t - T[l', l, k]]+ I[l', l, k]\right)$
   % \vspace{1pt}
    \STATE  $M[l, t] \leftarrow M[l^*, t - T[l^*, l, k^*]]+ I[l^*, l, k^*]$
     %\vspace{1pt}
    \STATE $A[l, t] \leftarrow A[l^*, t - T[l^*, l, k^*]] \cup \{l^*: l^* > 0\}$ %\vspace{1pt}
    \STATE $\widehat{C}_{l^* l k^*} \leftarrow $ compute via \cref{eq:ACij} %\vspace{1pt}
   % \STATE $\widetilde{C} \leftarrow \underset{C \subseteq (l^*, l]}{\mathrm{argmax}}\sum_{c \in C} \ell_1(\theta_c)$
  %  \STATE $\qquad~~\mathrm{subject~to~} k^* = 1 +\!\!\sum_{l\in C} (\mathrm{Ker}(\theta_l) - 1)$
    \STATE $C[l, t] \leftarrow C[l^*, t - T[l^*, l, k^*]] \cup \widehat{C}_{l^* l k^*}$
    \ENDFOR
    \ENDFOR
    \STATE $A^* \leftarrow A[L, T_0]$, $C^* \leftarrow C[L, T_0]$ 
    \STATE $k^*_i \leftarrow 1 + \sum_{l\in C^* \cap (a_{i-1}^*, a_{i}^*]} (\mathrm{Ker}(\theta_l) - 1), \forall i \in [|A^*|+1]$
    \OUTPUT $A^*$, $C^*$, and $(k_i^*)_{i=1}^{|A^*|+1}$.
    \end{algorithmic}
    \end{algorithm}
    \vspace{-2em}
\end{figure}

\begin{figure}[t]
    \centering
    \begin{algorithm}[H]
    \caption{LayerMerge}
    \label{alg:layermerge}
    \begin{algorithmic}
    \INPUT Input network $f$, latency budget $T_0$, descretization level $P$
    % sum of the kernel size $K_0 \coloneqq \sum_l \mathrm{Ker}(\theta_l)$
    \FOR{$i=0$ {\bfseries to} $L-1$} 
    \FOR{$j=i + 1$ {\bfseries to} $L$} 
    \FOR{$k \in K_{ij}$}%{$k=1$ {\bfseries to} $1 + \sum_{l=i+1}^j (\mathrm{Ker}(\theta_l) - 1)$} %K_0
    \STATE $\widetilde{A}_{ij}, \widetilde{C}_{ijk} \leftarrow $ compute via \Cref{eq:ACij}
    %= \{1, \ldots, i\} \cup \{j, \ldots, L\}$
    %\STATE $\widetilde{C}_{ijk} \leftarrow \underset{C \subseteq (i, j]}{\mathrm{argmax}}\sum_{l \in C} \ell_1(\theta_l)$
    %\STATE $\qquad~~\mathrm{subject~to~} k = 1 +\!\!\sum_{l\in C \cap (i, j]} (\mathrm{Ker}(\theta_l) - 1)$
    \STATE $\hat{\theta}_{ij}= \conv_{l=i\!+\!1}^{j} \left(\mathds{1}_{\widetilde{C}_{ijk}}(l) \theta_{l}+ \left(1-\mathds{1}_{\widetilde{C}_{ijk}}(l)\right)\theta_{\mathrm{id}} \right)$
    \STATE $I[i, j, k] \leftarrow $ compute via \Cref{eq:imp}
    \STATE $T[i, j, k] \leftarrow T(f_{\hat{\theta}_{ij}})$
    \ENDFOR
    \ENDFOR
    \ENDFOR
    \STATE $A^*, C^*, k^* \leftarrow$ \Cref{alg:dp}$\left(I, T, T_0, P\right)$
    \STATE Replace activation functions outside $A^*$ by $\mathrm{id}$ and convolution layers outside $C^*$ by $f_{\theta_{\mathrm{id}}}$ 
    %Replace layers in $f$ following set $A^*$ and $C^*$
    \STATE Fine-tune and merge the network
    \OUTPUT Merged network
    \end{algorithmic}
    \end{algorithm}
\end{figure}

\subsection{Dynamic programming algorithm}

Once we construct the lookup tables $T$ and $I$, we can obtain an exact solution of Problem \eqref{eq:dp_master} for discretized latency values, using dynamic programming (DP).
In particular, we discretize latency values in the lookup table $T$ by rounding them down to the closest values in $\{\frac{T_0}{P}, \frac{2T_0}{P} \ldots, T_0\}$, where $P$ is a large natural number that represents the discretization level.

Then, we consider a sub-problem of Problem \eqref{eq:dp_master} where we maximize over the first $\sblue{l}\in \{0, \ldots, L\}$ layers with latency budget $\sblue{t} \in \{\frac{T_0}{P}, \frac{2T_0}{P} \ldots, T_0\}$, as follows: 

{
\begin{align} \label{eq:dp_subprob} %k_i \in K_{0 L}
&\max_{A \subseteq [\sblue{l}-1], k_i} \sum_{i=1}^{|A|+1} I[a_{i-1}, a_{i}, k_{i}] \\
&\text{subject to } \sum_{i=1}^{|A|+1} T[a_{i-1}, a_i,k_i] < \sblue{t}, &\text{(latency constraint)} \nonumber\\
&~~~~~~~~~~~~~~~~~~~k_i \in K_{a_{i-1} a_i},
&\text{(merged kernel size)}\nonumber
\end{align}
}
% {\small
% \begin{align} \label{eq:dp_subprob} %k_i \in K_{0 L}
% &\underset{A \subseteq [\sblue{l}-1], k_i}{\mathrm{maximize}} ~\;\sum_{i=1}^{|A|+1} I[a_{i-1}, a_{i}, k_{i}] \\
% &\mathrm{subject\ to\ } ~\sum_{i=1}^{|A|+1} T[a_{i-1}, a_i,k_i] < \sblue{t}, &\text{(latency constraint)} \nonumber\\
% &~~~~~~~~~~~~~~~~~~~~~~k_i \in K_{a_{i-1} a_i},
% &\text{(merged kernel size)}\nonumber
% \end{align}
% }

where $a_0=0$ and $a_{|A|+1} = \sblue{l}$.
We define $M[l, t]$ as the corresponding maximum objective value achievable in the sub-problem \eqref{eq:dp_subprob}. Then $M[L, T_0]$ gives the maximum objective value achievable in Problem \eqref{eq:dp_master}.%, and set $M[0, t] = 0$ for any $t$. 
 We initialize $M[0, t] = 0$ for $t \geq 0$, and  $M[l, t] = - \infty$ for $t < 0$. 
Then, for $l>0$, the recurrence of the DP algorithm can be written as follows:

{
\begin{align} \label{eq:dp_recur}
M[l, t] = \max_{0 \leq l' < l,\,k \in K_{l' l}} \left(\underbrace{M[l', t - T[l', l, k]]}_{\text{Optimal importance sum until $l'$-layer}}\right. \nonumber\\
\left. + \underbrace{I[l', l, k]}_{\substack{\text{Importance value of the last compressed layer}}}\right).
\end{align}
}

We present the DP algorithm for Problem \eqref{eq:dp_master} in \Cref{alg:dp}. 
Once we compute the optimal sets $A^*$ and $C^*$, we fine-tune the network after replacing the layers accordingly.
Then, we merge every convolution layer between $a_{i-1}^*$ and $a_{i}^*$ for all $i \in [|A|+1]$ at inference time.
We outline the overall procedure of LayerMerge in \Cref{alg:layermerge}.

\subsection{Theoretical analysis}

We show that the proposed DP algorithm optimally solves the surrogate optimization Problem \eqref{eq:dp_master}. The proof is given in \Cref{app:proof}.

\begin{restatable}{theorem}{DPoptimal} 
The solution $A^*$ and $(k_i^*)_{i=1}^{|A^*|+1}$ given by \Cref{alg:dp} is an optimal solution of Problem \eqref{eq:dp_master}.
\label{prop:dp}
\end{restatable}

The time complexity of the DP algorithm is $O\left(L^2 P K_0 \right)$.
In practice, the DP algorithm is highly efficient, typically completing within a few seconds on CPU.
Furthermore, our method efficiently computes the DP lookup tables, exploiting the structure of the problem.
It is worth noting that this can be done in an embarrassingly parallel fashion.
We report the wall-clock time for constructing the DP lookup tables 
%a practical measurement cost of our method 
in \Cref{app:cost}.
% Compared to the exhaustive search, the proposed DP formulation reduces the measurement cost drop from $O\left(2^{2L}\right)$ to $O\left(L^2 K_0\right)$.

\section{Experiments}
In this section, we provide experimental results demonstrating the effectiveness of our method across different network architectures and tasks.
We apply our method on ResNet-34 and MobileNetV2 models \citep{resnet,mobilenetv2} for the image classification task, and on the DDPM model \citep{ddpm} for the image generation task.
We present additional details on how we handle normalization layers, strided convolutions, padding, skip connections, and other network-specific implementation details of our method in \Cref{app:detail}.

\paragraph{Baselines}
We compare our method to a depth compression method and a layer pruning method since both approaches rely, like our method, on reducing the number of layers of the network to accelerate it. We also include a comparison with a knowledge distillation method in \cref{app:AddExps}.
For the depth compression baseline, we use the state-of-the-art work of \citet{kim23efficient}, which we denote as \textit{Depth}.

Existing layer pruning methods do not directly take latency into consideration during pruning \citep{discriminative,shallowing,layerprune}.  
To address this gap, we propose a variant of our method tailored specifically for layer pruning, which we use as our layer pruning baseline.

%Regarding the layer pruning baseline, we note that there are no existing methods that prune the layer in a latency-aware manner, to the best of our knowledge \citep{discriminative,shallowing,layerprune}.
% Concretely, we formulate the following subset selection problem for the layer pruning:

% {\footnotesize
% \begin{align*}
%     &\underset{C\subseteq [L]}{\mathrm{maximize}}~~~\underset{\theta}{\mathrm{max}}\;\mathrm{Perf}\left( \bigO_{l=1}^L \biggl(\sigma_{l} \circ \underbrace{f_{C, \theta_l, l} }_{\text{Replaced conv}}\biggr)\right) \nonumber\\
%     &\mathrm{subject\ to\ } \nonumber\\
% &~~~~~~~R \subseteq C, ~~~~~~~~~~~~~~~~~~~~~~~~~~~~~~~~~~~~~~~~~~~\text{(irreducible convolution)} \nonumber\\
%     &~~~~~~~f_{C, \theta_l, l} = \left(\mathds{1}_C(l) f_{\theta_{l}} + \left(1-\mathds{1}_C(l)\right)f_{\theta_\mathrm{id}} \right),~~\text{(replaced layers)}\nonumber\\
% &~~~~~~~T\left( \bigO_{l=1}^{L} \left( \sigma_{l} \circ f_{C, \theta_l, l}\right) \right) < T_0.~~~~~~~~~~~~~~~~~~\text{(latency constraint)}\nonumber
% \end{align*}
% }

Specifically, we assign an importance value $I[l]$ and a latency cost $T[l]$ for each convolution layer $l \in [L]$.
The importance value of the convolution layer is defined as the change in performance resulting from replacing the layer with the identity layer in the original network and then fine-tuning it for a few steps.
As in \textit{LayerMerge}, we approximate the overall latency and importance of the network by the sum of the layer-wise latencies and importance values, respectively.
We solve the following surrogate problem:
\begin{align} \label{eq:dplayer}
& \max_{R \subseteq C\subseteq [L]} ~\sum_{l \in C} I[l] 
\\
& \text{subject to } \sum_{l \in C} T[l] < T_0. & \text{(latency constraint)} \nonumber
\end{align}

% {\normalsize
% \begin{align} \label{eq:dplayer}
% &\underset{R \subseteq C\subseteq [L]}{\mathrm{maximize}} ~~~ \sum_{l \in C} I[l] 
% \\
% &\mathrm{subject\ to\ } ~~\sum_{l \in C} T[l] < T_0.~~\text{(latency constraint)} \nonumber \nonumber
% \end{align}
% }

Problem \eqref{eq:dplayer} is a 0-1 knapsack problem that can be solved exactly for discretized latency values via a DP algorithm in $O(L P)$ time, where $P$ denotes the discretization level.
We denote this method as \textit{LayerOnly}.

We additionally compare with a channel pruning baseline for each network.
Note that channel pruning is an orthogonal approach to depth compression.
Nonetheless, we include channel pruning results to study the effectiveness of reducing the number of layers compared to reducing the width in different types of networks.
We compare with HALP \citep{halp} on ResNet-34, with AMC and MetaPruning \citep{amc,metapruning} on MobileNetV2, and with Diff-Pruning \citep{diffpruning} on DDPM.

For ResNet-34, we apply the depth compression and the channel pruning baselines using their publicly available code \citep{halp,kim23efficient}.
%re-implement 
% in ResNet-34, 
We also use the pre-trained weights of the original network from \citet{timm} for all compression methods.
For MobileNetV2, we report the results for the depth compression baseline from the original paper \citep{kim23efficient}.
For the channel pruning baselines, we prune channels of each layer using the same channel ratio of their optimized model from their open-sourced code \citep{amc,metapruning}.
We use the pre-trained weights of the original network from the public code of \citet{kim23efficient} to ensure a fair comparison. % of compression results in MobileNetV2.
For DDPM, we apply the depth compression and the channel pruning baselines using their open-sourced code \citep{kim23efficient,diffpruning}.
We use the pre-trained weights of the original network from \citet{ddim} for all compression methods in DDPM.
In all networks, we compare the compression results in each table using an identical fine-tuning schedule for fair comparison.

Finally, we provide ablation studies on the importance of joint optimization on activation layers and convolution layers in our method, where we compare our method to sequentially applying {Depth} then {LayerOnly}.
Throughout this section, we refer to each compressed model obtained by Depth, LayerOnly, and LayerMerge as Depth-$p$\%, LayerOnly-$p$\%, and LayerMerge-$p$\%, respectively. 
Here, $p$\% is calculated as ${T_0}/{T_\mathrm{orig}}$, where $T_0$ is the chosen latency budget and $T_\mathrm{orig}$ is the latency of the original model.

\paragraph{Experimental setup}
We construct the latency lookup table of each method on RTX2080 Ti GPU and report the wall-clock latency speed-up of the compressed networks measured on the same device.
We provide the details on the measurement process in \Cref{app:cost}.
Notably, we measure the latency of the network in two different formats, PyTorch format and TensorRT compiled format \citep{pytorch,tensorrt}.
When measuring latency speedup, we use a batch size of 128 for the ImageNet dataset and the CIFAR10 dataset, following the same measurement protocol from \citet{kim23efficient,halp,diffpruning}.
For ResNet-34, we fine-tune each pruned network for 90 epochs following the same fine-tuning recipe as HALP \citep{halp}.
For MobileNetV2, we fine-tune for 180 epochs, using the same fine-tuning recipe as \citet{kim23efficient}.
For DDPM, we follow the fine-tuning and sampling recipe of Diff-Pruning \citep{diffpruning}, except for the learning rate which we set to $4 \times 10^{-4}$ since it leads to better performance. We present a representative subset of the results in Tables \ref{tab:rn34} to \ref{tab:abl}. For additional results, see \cref{app:AddExps}.

\subsection{Classification task results}
In this section, we evaluate the performance of the different pruning methods on ResNet-34, MobileNetV2-1.0, and MobileNetV2-1.4 models, on the ImageNet dataset \citep{resnet,mobilenetv2,imagenet}.
% In our implementation, we use the public pre-trained weights from  \citet{timm} following the practice of \citet{depthshrinker,kim23efficient}.
We report the last top-1 accuracy of the compressed model after fine-tuning, evaluated on the validation set, and its corresponding latency speedup.
\begin{table}[t]
\caption{Accuracy and latency speed-up of applying compression methods to ResNet-34 on ImageNet dataset. The latency speed-up is measured on RTX2080 Ti GPU at batch size 128.}
\vspace{0.5em}
\centering
\begin{adjustbox}{max width=1.0\columnwidth}
\begin{tabular}{lcccc}
\toprule
 && \normalsize{PyTorch}& \normalsize{TensorRT}\\
 Network     & Acc (\%) $\uparrow$& Speed-up $\uparrow$ & Speed-up $\uparrow$\\
 \cmidrule(r){1-2}\cmidrule(r){3-3} \cmidrule(r){4-4}
    ResNet-34  & 74.42 & 1.00$\times$ & 1.00$\times$ \\
 \cmidrule(r){1-2}\cmidrule(r){3-3} \cmidrule(r){4-4}
    HALP-80\%  \citep{halp} & 73.98 & 1.23$\times$ & \textbf{1.25}$\times$ \\
    Depth-78\% \citep{kim23efficient} & 73.49 & 1.24$\times$ & 1.14$\times$ \\
    LayerOnly-73\% (Ours) & 74.06 & 1.33$\times$ & 1.24$\times$ \\
    LayerMerge-71\% (Ours)& \textbf{74.26}  & \textbf{1.36}$\times$ & \textbf{1.25}$\times$ \\
 \cmidrule(r){1-2}\cmidrule(r){3-3} \cmidrule(r){4-4}
    HALP-65\% \citep{halp} & 73.36 & 1.48$\times$ & 1.45$\times$ \\
    Depth-68\% \citep{kim23efficient}& 73.35 & 1.40$\times$ & 1.26$\times$ \\
    LayerOnly-59\% (Ours)       & 73.31 & \textbf{1.65}$\times$ & 1.48$\times$ \\
    LayerMerge-60\% (Ours)      & \textbf{73.46} & 1.56$\times$ & \textbf{1.50}$\times$ \\
 \cmidrule(r){1-2}\cmidrule(r){3-3} \cmidrule(r){4-4}
    HALP-55\% \citep{halp} & 72.69 & 1.69$\times$ & 1.69$\times$ \\
    Depth-63\% \citep{kim23efficient}& 72.33 & 1.43$\times$ & 1.24$\times$ \\
    LayerOnly-49\% (Ours) & 72.58 & \textbf{1.82}$\times$ & 1.64$\times$ \\
    LayerMerge-50\% (Ours) & \textbf{72.84} & 1.79$\times$ & \textbf{1.72}$\times$ \\
\bottomrule
\end{tabular}
\end{adjustbox}
\label{tab:rn34}
\end{table}

\paragraph{ResNet-34}
\Cref{tab:rn34} summarizes the different compression results on ResNet-34. 
HALP-$p$\% refers to the pruned model obtained by HALP by setting the latency budget to be $p$\% of the original model latency.
LayerMerge outperforms existing channel pruning and depth compression baselines.
Specifically, we achieve 1.10$\times$ speed-up in PyTorch with 0.77\% point higher accuracy compared to the depth compression baseline (comparing LayerMerge-71\% to Depth-78\%).
It is worth noting that the layer pruning variant of our method performs on par with LayerMerge on ResNet-34.
This is mainly due to ResNet-34 being more suitable for layer pruning than depth compression.
Indeed, LayerMerge frequently opts for pruning convolution layers over activation functions when applied to ResNet-34.

\paragraph{MobileNetV2}

\Cref{tab:mbv2-1.0} presents the various compression results on MobileNetV2-1.0.
AMC-$p$\% denotes the pruned model obtained by AMC by setting the FLOPs budget to be $p$\% of the original model FLOPs.
LayerMerge surpasses existing channel pruning and depth compression baselines, as well as our layer pruning variant.
In particular, we achieve 1.63$\times$ speed-up in PyTorch without losing accuracy from the original network (LayerMerge-55\%).  
\Cref{tab:mbv2-1.4} further shows the compression results on MobileNetV2-1.4.
LayerMerge outperforms existing methods, offering 0.23\% point higher accuracy with a larger speed-up compared to the depth compression baseline (comparing LayerMerge-43\% to Depth-62\%).
The gain in efficiency stems from its unique capability to jointly prune layers and merge them, as demonstrated in \Cref{fig:qual}.

\begin{table}[t]
\caption{Accuracy and latency speed-up of applying compression methods to MobileNetV2-1.0 on ImageNet dataset. The latency speed-up is measured on RTX2080 Ti GPU at batch size 128.}
\vspace{0.5em}
\centering
\begin{adjustbox}{max width=1.0\columnwidth}
\begin{tabular}{lcccc}
\toprule
 && \normalsize{PyTorch}& \normalsize{TensorRT}\\
 Network     & Acc (\%) $\uparrow$& Speed-up $\uparrow$ & Speed-up $\uparrow$\\
 \cmidrule(r){1-2}\cmidrule(r){3-3} \cmidrule(r){4-4}
    MobileNetV2-1.0  & 72.89 & 1.00$\times$ & 1.00$\times$ \\
 \cmidrule(r){1-2}\cmidrule(r){3-3} \cmidrule(r){4-4}
    AMC-70\% \citep{amc} & 72.01 & 1.32$\times$ & 1.34$\times$ \\
    Depth-74\% \citep{kim23efficient} & 72.83 & 1.62$\times$ & \textbf{1.42}$\times$ \\
    LayerOnly-73\% (Ours) & 69.66 & 1.30$\times$ & 1.35$\times$\\
    LayerMerge-55\% (Ours)& \textbf{72.99}  & \textbf{1.63}$\times$ & \textbf{1.42}$\times$ \\
 \cmidrule(r){1-2}\cmidrule(r){3-3} \cmidrule(r){4-4}
    Depth-66\% \citep{kim23efficient} & 72.13 & 1.88$\times$ & 1.57$\times$ \\
    LayerMerge-46\% (Ours) & \textbf{72.46} & \textbf{1.90}$\times$ & \textbf{1.65}$\times$ \\
 \cmidrule(r){1-2}\cmidrule(r){3-3} \cmidrule(r){4-4}
    Depth-59\% \citep{kim23efficient} & 71.44 & 2.07$\times$ & 1.79$\times$ \\
    LayerMerge-38\% (Ours) & \textbf{71.74} & \textbf{2.18}$\times$ & \textbf{1.84}$\times$ \\
 \cmidrule(r){1-2}\cmidrule(r){3-3} \cmidrule(r){4-4}
    Depth-53\% \citep{kim23efficient} & 70.65 & 2.47$\times$ & 1.97$\times$ \\
    LayerMerge-33\% (Ours) & \textbf{70.99} & \textbf{2.49}$\times$ & \textbf{2.05}$\times$ \\
\bottomrule
\end{tabular}
\end{adjustbox}
\label{tab:mbv2-1.0}
\end{table}

\begin{table}[t]
\caption{Accuracy and latency speed-up of applying compression methods to MobileNetV2-1.4 on ImageNet dataset. The latency speed-up is measured on RTX2080 Ti GPU at batch size 128.}
\vspace{0.5em}
\centering
\begin{adjustbox}{max width=1.0\columnwidth}
\begin{tabular}{lcccc}
\toprule
 && \normalsize{PyTorch}& \normalsize{TensorRT}\\
 Network     & Acc (\%) $\uparrow$& Speed-up $\uparrow$ & Speed-up $\uparrow$\\
 \cmidrule(r){1-2}\cmidrule(r){3-3} \cmidrule(r){4-4}
    MobileNetV2-1.4  & 76.28 & 1.00$\times$ & 1.00$\times$ \\
 \cmidrule(r){1-2}\cmidrule(r){3-3} \cmidrule(r){4-4}
    MetaPruning-1.0$\times$ \citep{metapruning} & 73.69 & 1.59$\times$ & 1.38$\times$ \\
    Depth-62\% \citep{kim23efficient} & 74.68 & 1.93$\times$ & \textbf{1.61}$\times$ \\
    LayerOnly-75\% (Ours) & 73.94 & 1.27$\times$ & 1.28$\times$ \\
    LayerMerge-43\% (Ours)& \textbf{74.91}  & \textbf{1.99}$\times$ & \textbf{1.61}$\times$ \\
 \cmidrule(r){1-2}\cmidrule(r){3-3} \cmidrule(r){4-4}
    Depth-60\% \citep{kim23efficient} & 74.19 & 1.99$\times$ & 1.67$\times$ \\
    LayerMerge-42\% (Ours) & \textbf{74.48} & \textbf{2.07}$\times$ & \textbf{1.73}$\times$ \\
 \cmidrule(r){1-2}\cmidrule(r){3-3} \cmidrule(r){4-4}
    Depth-53\% \citep{kim23efficient} & 73.46 &	2.27$\times$ & 1.85$\times$  \\
    LayerMerge-35\% (Ours) & \textbf{73.99} & \textbf{2.39}$\times$ & \textbf{1.93}$\times$ \\
 \cmidrule(r){1-2}\cmidrule(r){3-3} \cmidrule(r){4-4}
    Depth-46\% \citep{kim23efficient}& 72.57 & 2.41$\times$ & 2.01$\times$ \\
    LayerMerge-30\% (Ours) & \textbf{73.29} & \textbf{2.72}$\times$ & \textbf{2.12}$\times$ \\
\bottomrule
\end{tabular}
\end{adjustbox}
\label{tab:mbv2-1.4}
\end{table}

\subsection{Generation task results}

\begin{table}[t]
\caption{FID metric and PyTorch latency speed-up of compression methods applied to DDPM on CIFAR10 dataset. The latency speed-up is measured on RTX2080 Ti GPU at batch size 128.}
\vspace{0.5em}
\centering
\begin{adjustbox}{max width=1.0\columnwidth}
\begin{tabular}{lcccc}
\toprule
 && \normalsize{PyTorch} & Fine-tune\\
 Network     &  FID $\downarrow$ & Speed-up $\uparrow$ & Steps $\downarrow$\\
 \cmidrule(r){1-2}\cmidrule(r){3-3} \cmidrule(r){4-4}
    DDPM  & 4.18 & 1.00$\times$ & - \\
 \cmidrule(r){1-2}\cmidrule(r){3-3} \cmidrule(r){4-4}
    Depth-89\% \citep{kim23efficient}& 4.21 &  1.04$\times$ & 100K \\
    LayerOnly-77\% (Ours) & 4.64 & 1.09$\times$ & 100K\\
    LayerMerge-73\% (Ours)& \textbf{4.16}  & \textbf{1.13}$\times$ & 100K\\
 \cmidrule(r){1-2}\cmidrule(r){3-3} \cmidrule(r){4-4}
    Depth-85\% \citep{kim23efficient}& 4.78 & 1.08$\times$ & 100K \\
    LayerOnly-69\% (Ours) & 5.52 & 1.14$\times$ & 100K \\
    LayerMerge-70\% (Ours) & \textbf{4.55} & \textbf{1.16}$\times$ & 100K \\
 \cmidrule(r){1-2}\cmidrule(r){3-3} \cmidrule(r){4-4}
     % Diff-30\% \citep{diffpruning} & \textbf{4.85 }& \textbf{1.40}$\times$ & 100K \\
    LayerOnly-54\% (Ours) & 6.23 &  1.26$\times$ & 100K \\
    LayerMerge-58\% (Ours) & \textbf{5.61} & \textbf{1.27}$\times$ & 100K  \\
 % \cmidrule(r){1-2}\cmidrule(r){3-3} \cmidrule(r){4-4}
    % Diff-60\% \citep{diffpruning} & 7.90 & 2.33$\times$ & 100K \\
\bottomrule
\end{tabular}
\end{adjustbox}
\label{tab:ddpm}
\end{table}

\begin{table}[t]
\caption{FID metric and PyTorch latency speed-up of compression methods applied to channel pruned DDPM on CIFAR10 dataset. The latency speed-up is measured on RTX2080 Ti GPU at batch size 128. Diff-$p$\% denotes applying Diff-Pruning with $p$\% compression ratio.}
\vspace{0.5em}
\centering
\begin{adjustbox}{max width=1.0\columnwidth}
\begin{tabular}{lcccc}
\toprule
 && \normalsize{PyTorch} & Fine-tune\\
 Network     &  FID $\downarrow$ & Speed-up $\uparrow$ & Steps $\downarrow$\\
 \cmidrule(r){1-2}\cmidrule(r){3-3} \cmidrule(r){4-4}
    DDPM  & 4.18 & 1.00$\times$ & - \\
 \cmidrule(r){1-2}\cmidrule(r){3-3} \cmidrule(r){4-4}
    Diff-30\% \citep{diffpruning} & 4.85 & 1.40$\times$ & 100K \\
    Diff-60\% \citep{diffpruning} & 7.90 & 2.33$\times$ & 100K \\
 \cmidrule(r){1-2}\cmidrule(r){3-3} \cmidrule(r){4-4}
    Diff-70\% \citep{diffpruning}&  9.89 & 2.57$\times$ & 200K \\
    Diff-60\% $\rightarrow$ Depth-86\% \citep{kim23efficient}& 9.09 & 2.42$\times$ & 200K \\
    Diff-60\% $\rightarrow$ LayerOnly-63\% (Ours) & 10.15 & \textbf{2.59}$\times$ & 200K \\
    Diff-60\% $\rightarrow$ LayerMerge-70\% (Ours)& \textbf{8.92} & \textbf{2.59}$\times$ & 200K \\
\bottomrule
\end{tabular}
\end{adjustbox}
\label{tab:diffprune}
\end{table}

In this section, we evaluate the performance of the different pruning methods on DDPM model \citep{ddpm} on the CIFAR10 dataset \citep{cifar10}.
% We use the public pre-trained weights from \citet{ddim} following \citet{diffpruning}.
We measure performance using the standard Frechet Inception Distance (FID) metric \citep{FID}. 
We report the last FID of the compressed model after fine-tuning, evaluated on the validation set, and its corresponding latency speedup.

\paragraph{DDPM}

\Cref{tab:ddpm} reports the different compression results on DDPM.
LayerMerge shows superior performance compared to the existing depth compression baseline and the layer pruning variant of our method.
Specifically, we achieve 1.08$\times$ speed-up with a lower FID metric compared to the depth compression baseline (comparing LayerMerge-73\% to Depth-89\%).
The channel pruning baseline Diff-Pruning shows superior performance than LayerMerge here (\Cref{tab:diffprune}).
This is likely due to DDPM having more channel redundancy than other models.

\paragraph{Channel pruned DDPM}
We note that channel pruning and depth compression methods can be jointly applied.
We thus include results where we apply our method to the channel pruned DDPM model obtained by Diff-Pruning in \Cref{tab:diffprune} \citep{diffpruning}.
Diff-$p$\% denotes the pruned model obtained with Diff-Pruning by removing $p$\% of the channels in each layer.
The results show that combining our method with Diff-Pruning achieves a larger speed-up than solely relying on Diff-Pruning, and our method consistently outperforms the depth compression and layer pruning baselines in this setting as well.

% \paragraph{Channel pruned DDPM: Diff-Pruning-60\%}
% We note that channel pruning and depth compression methods can be jointly applied.
% We thus include results where we apply our method to the DDPM pruned model obtained by Diff-Pruning with a 60\% compression ratio in \Cref{tab:diffprune}.
% The results show that combining our method with Diff-Pruning achieves a larger speed-up than solely relying on Diff-Pruning, and our method consistently outperforms the depth compression and layer pruning baselines in this setting, too.

\subsection{Ablation studies}

\begin{table}[t]
\caption{Accuracy and corresponding latency speed-up compared to sequential optimization and our method evaluated with MobileNetV2-1.0 on ImageNet dataset.}
\vspace{0.5em}
\centering
\begin{adjustbox}{max width=1.0\columnwidth}
\begin{tabular}{lcccc}
\toprule
 && \normalsize{PyTorch}& \normalsize{TensorRT}\\
 Network     & Acc (\%) $\uparrow$& Speed-up $\uparrow$ & Speed-up $\uparrow$\\
 \cmidrule(r){1-2}\cmidrule(r){3-3} \cmidrule(r){4-4}
    MobileNetV2-1.0  & 72.89 & 1.00$\times$ & 1.00$\times$ \\
 \cmidrule(r){1-2}\cmidrule(r){3-3} \cmidrule(r){4-4}
    Depth-74\% $\rightarrow$ LayerOnly-72\% & 71.72 & 2.09$\times$ & 1.79$\times$\\
    % Layer $\rightarrow$ Depth &  &  & \\
    LayerMerge-39\% (Ours)& \textbf{71.89}  & \textbf{2.15}$\times$ & \textbf{1.80}$\times$ \\
 \cmidrule(r){1-2}\cmidrule(r){3-3} \cmidrule(r){4-4}
    Depth-74\% $\rightarrow$ LayerOnly-64\% & 70.14 & 2.33$\times$ & 2.09$\times$\\
    % Layer $\rightarrow$ Depth &  &  & \\
    LayerMerge-33\% (Ours)& \textbf{70.81}  & \textbf{2.47}$\times$ & \textbf{2.15}$\times$ \\
\bottomrule
\end{tabular}
\end{adjustbox}
\label{tab:abl}
\end{table}

Our method jointly optimizes the selection of the activation layers and convolution layers to prune.
An alternative way to do this is to sequentially optimize the selection of each type of layer independently.
We compare our method to sequentially applying {Depth} then {LayerOnly} on MobileNetV2 on the ImageNet dataset in \Cref{tab:abl}.
Our method outperforms the sequential optimization baseline, underlining the importance of our joint optimization approach.
We provide the details in \Cref{app:exp}.

\section{Related Work}

\paragraph{Unstructured pruning}
Unstructured pruning methods remove individual neurons in the network to achieve network sparsity \citep{unstructured,transposable,chita,sparsegpt}.
The closest method to ours in this line of work is \citet{spdy}, which proposes a dynamic programming algorithm that determines layer-wise sparsity under a given latency constraint.
However, unstructured pruning methods require specialized hardware to achieve computational savings.

\paragraph{Channel pruning}
In contrast, structured pruning methods, which consist of removing redundant regular regions of weights, can yield computational savings on off-the-shelf hardware.
Among such methods are channel pruning methods, which remove redundant channels in a convolutional neural network \citep{taylor16,taylor17,pfec}.
\citet{knapsack} formulate this as a knapsack problem that maximizes the sum of channel importance values under a given FLOPs budget. 
Similarly, \citet{halp} formulates another knapsack problem, which maximizes the sum of channel importance values under a latency constraint on a target device.

\paragraph{Layer pruning}
Layer pruning methods aim to make a shallower network by entirely removing certain convolution layers \citep{discriminative,shallowing, layerprune}. 
However, their aggressive nature in removing parameters results in a large performance degradation under high compression ratios \citep{depthshrinker}.

\paragraph{Depth compression}
Instead, depth compression methods focus on eliminating unimportant non-linear activation layers, then merging consecutive convolution layer to reduce the network's depth \citep{layerfolding,depthshrinker,kim23efficient}.
In particular, \citet{layerfolding} and \citet{depthshrinker} propose to train a soft parameter that controls the intensity of non-linearity of each layer with an additional loss that penalizes the absolute value of the soft parameter.
More recently, \citet{kim23efficient} propose to maximize the sum of the importance values of merged layers under a latency constraint, via a dynamic programming algorithm.
However, this line of work suffers from a fundamental drawback as merging layers leads to an increase in the kernel size of the merged layers. \citet{depthshrinker} sidestep this issue, as they only consider merging within inverted residual blocks.
However, this restriction not only limits the applicability of their method to mobile-optimized CNNs, but also limits its performance. Indeed, \citet{kim23efficient} have shown that their method outperforms that of \cite{depthshrinker}.

%Our method effectively addresses this challenge by jointly optimizing the selection of the convolution layer and the non-linear activation layer.

% Concretely, let $A$ be the layer indices that keeps the non-linear activation layer and $S$ be the layer indices that do not merge with the subsequent layer.
% Then, they propose to optimize $A$ and $S$ as
% \begin{align*}
%     \label{eq:kim23efficient}
% &\underset{A\subseteq S \subseteq \{0, 1, \ldots, L\}}{\mathrm{maximize}} ~~\sum_{a_{j\!-\!1},a_j \in  A } I[a_{j\!-\!1}, a_{j}]\\
% &~~~~\mathrm{subject\ to\ } ~~~~~ \sum_{s_{i\!-\!1},s_i \in S} T[s_{i\!-\!1}, s_{i}] < T_0, \nonumber
% \end{align*}
% where $I[i, j]$ and $T[i, j]$ represent the importance value and the latency cost of the layer merged from $(i+1)$-th layer to the $j$-th layer. Here, $T_0$ denotes the latency constraint.

\section{Conclusion}
%Existing depth compression methods have a fundamental limitation as the kernel size of the merged layers becomes larger, which significantly deteriorates the latency speed-up gained from reducing the depth of the network.
%We demonstrate that this problem can be addressed by jointly pruning convolution layers and activation functions and 
We propose \textit{LayerMerge}, a novel efficient depth compression method, which jointly prunes convolution layers and activation functions to achieve a desired target latency while minimizing the performance loss.
Our method avoids the problem of increasing kernel size in merged layers, which existing depth compression methods suffer from.
It consistently outperforms existing depth compression and layer pruning baselines in various settings.
%optimizing the selection of both types of layers to remove.
%We formulate a novel surrogate optimization problem and efficiently solve it via a dynamic programming algorithm.
%on various network architectures and is evaluated on both image classification and generation tasks.

%\section*{Acknowledgements}
% \textbf{Do not} include acknowledgements in the initial version of
% the paper submitted for blind review.

% If a paper is accepted, the final camera-ready version can (and
% probably should) include acknowledgements. In this case, please
% place such acknowledgements in an unnumbered section at the
% end of the paper. Typically, this will include thanks to reviewers
% who gave useful comments, to colleagues who contributed to the ideas,
% and to funding agencies and corporate sponsors that provided financial

\section*{Impact Statement}

\looseness=-1 This work contributes to the area of NN compression, in particular to compressing pre-trained CNNs and diffusion models. As such, it helps reduce the energy consumption and computational resources of such models at inference. Hence, this work contributes to reducing the environmental impact of NNs and enabling their use on resource-constrained devices like mobile phones and for latency-critical applications such as self-driving cars. On the other hand, pruning has been shown to have a disparate impact on performance between different sub-groups of data, which amplifies existing algorithmic bias \cite{hooker2020bias, paganini2020prune}. There is an ongoing effort to mitigate this negative impact of pruning either using fairness-aware pruning \cite{lin2022fairgrape} or by modifying the objective during fine-tuning of the pruned model \cite{tran2022pruning}. The latter approach can be applied to our pruning method \cite{hashemizadeh2023balancing}.
%On the other hand, reducing the cost of using these models also lowers the barrier to their use by malicious actors.

\section*{Acknowledgements}
This work was supported by Samsung Advanced Institute of Technology, Samsung Electronics Co., Ltd. (IO220810-01900-01), Institute of Information \& Communications Technology Planning \& Evaluation (IITP) grant funded by the Korea government (MSIT) [No. RS-2020-II200882, (SW STAR LAB) Development of deployable learning intelligence via self-sustainable and trustworthy machine learning and No. RS-2021-II211343, Artificial Intelligence Graduate School Program (Seoul National University)], and Basic Science Research Program through the National Research Foundation of Korea (NRF) funded by the Ministry of Education (RS-2023-00274280).
Hyun Oh Song is the corresponding author.


\bibliography{main}

\begin{thebibliography}{46}
\providecommand{\natexlab}[1]{#1}
\providecommand{\url}[1]{\texttt{#1}}
\expandafter\ifx\csname urlstyle\endcsname\relax
  \providecommand{\doi}[1]{doi: #1}\else
  \providecommand{\doi}{doi: \begingroup \urlstyle{rm}\Url}\fi

\bibitem[Aflalo et~al.(2020)Aflalo, Noy, Lin, Friedman, and Zelnik]{knapsack}
Aflalo, Y., Noy, A., Lin, M., Friedman, I., and Zelnik, L.
\newblock Knapsack pruning with inner distillation.
\newblock \emph{arXiv preprint arXiv:2002.08258}, 2020.

\bibitem[Benbaki et~al.(2023)Benbaki, Chen, Meng, Hazimeh, Ponomareva, Zhao, and Mazumder]{chita}
Benbaki, R., Chen, W., Meng, X., Hazimeh, H., Ponomareva, N., Zhao, Z., and Mazumder, R.
\newblock Fast as chita: Neural network pruning with combinatorial optimization.
\newblock In \emph{ICML}, 2023.

\bibitem[Cai et~al.(2019)Cai, Zhu, and Han]{proxylessnas}
Cai, H., Zhu, L., and Han, S.
\newblock Proxylessnas: Direct neural architecture search on target task and hardware, 2019.

\bibitem[Chen et~al.(2018{\natexlab{a}})Chen, Zhu, Papandreou, Schroff, and Adam]{deeplabv3plus}
Chen, L.-C., Zhu, Y., Papandreou, G., Schroff, F., and Adam, H.
\newblock Encoder-decoder with atrous separable convolution for semantic image segmentation.
\newblock In \emph{ECCV}, 2018{\natexlab{a}}.

\bibitem[Chen \& Zhao(2018)Chen and Zhao]{shallowing}
Chen, S. and Zhao, Q.
\newblock Shallowing deep networks: Layer-wise pruning based on feature representations.
\newblock \emph{IEEE transactions on pattern analysis and machine intelligence}, 2018.

\bibitem[Chen et~al.(2018{\natexlab{b}})Chen, Mishra, Rohaninejad, and Abbeel]{chen2018pixelsnail}
Chen, X., Mishra, N., Rohaninejad, M., and Abbeel, P.
\newblock Pixelsnail: An improved autoregressive generative model.
\newblock In \emph{ICML}, 2018{\natexlab{b}}.

\bibitem[Ding et~al.(2021)Ding, Zhang, Ma, Han, Ding, and Sun]{repvgg}
Ding, X., Zhang, X., Ma, N., Han, J., Ding, G., and Sun, J.
\newblock Repvgg: Making vgg-style convnets great again.
\newblock In \emph{CVPR}, 2021.

\bibitem[Dror et~al.(2022)Dror, Zehngut, Raviv, Artyomov, Vitek, and Jevnisek]{layerfolding}
Dror, A.~B., Zehngut, N., Raviv, A., Artyomov, E., Vitek, R., and Jevnisek, R.
\newblock Layer folding: Neural network depth reduction using activation linearization.
\newblock In \emph{BMVC}, 2022.

\bibitem[Elkerdawy et~al.(2020)Elkerdawy, Elhoushi, Singh, Zhang, and Ray]{layerprune}
Elkerdawy, S., Elhoushi, M., Singh, A., Zhang, H., and Ray, N.
\newblock To filter prune, or to layer prune, that is the question.
\newblock In \emph{ACCV}, 2020.

\bibitem[Fang et~al.(2023)Fang, Ma, and Wang]{diffpruning}
Fang, G., Ma, X., and Wang, X.
\newblock Structural pruning for diffusion models.
\newblock In \emph{NeurIPS}, 2023.

\bibitem[Frantar \& Alistarh(2022)Frantar and Alistarh]{spdy}
Frantar, E. and Alistarh, D.
\newblock Spdy: Accurate pruning with speedup guarantees.
\newblock In \emph{ICML}, 2022.

\bibitem[Frantar \& Alistarh(2023)Frantar and Alistarh]{sparsegpt}
Frantar, E. and Alistarh, D.
\newblock {S}parse{GPT}: Massive language models can be accurately pruned in one-shot.
\newblock In \emph{ICML}, 2023.

\bibitem[Fu et~al.(2022)Fu, Yang, Yuan, Li, Wan, Krishnamoorthi, Chandra, and Lin]{depthshrinker}
Fu, Y., Yang, H., Yuan, J., Li, M., Wan, C., Krishnamoorthi, R., Chandra, V., and Lin, Y.
\newblock Depthshrinker: A new compression paradigm towards boosting real-hardware efficiency of compact neural networks.
\newblock In \emph{ICML}, 2022.

\bibitem[Girshick(2015)]{fastrcnn}
Girshick, R.
\newblock Fast r-cnn.
\newblock In \emph{ICCV}, 2015.

\bibitem[Han et~al.(2015)Han, Pool, Tran, and Dally]{unstructured}
Han, S., Pool, J., Tran, J., and Dally, W.
\newblock Learning both weights and connections for efficient neural network.
\newblock In \emph{NeurIPS}, 2015.

\bibitem[Hashemizadeh et~al.(2024)Hashemizadeh, Ramirez, Sukumaran, Farnadi, Lacoste-Julien, and Gallego-Posada]{hashemizadeh2023balancing}
Hashemizadeh, M., Ramirez, J., Sukumaran, R., Farnadi, G., Lacoste-Julien, S., and Gallego-Posada, J.
\newblock Balancing act: Constraining disparate impact in sparse models.
\newblock In \emph{ICLR}, 2024.

\bibitem[He et~al.(2016)He, Zhang, Ren, and Sun]{resnet}
He, K., Zhang, X., Ren, S., and Sun, J.
\newblock Deep residual learning for image recognition.
\newblock In \emph{CVPR}, 2016.

\bibitem[He et~al.(2018)He, Lin, Liu, Wang, Li, and Han]{amc}
He, Y., Lin, J., Liu, Z., Wang, H., Li, L.-J., and Han, S.
\newblock Amc: Automl for model compression and acceleration on mobile devices.
\newblock In \emph{ECCV}, 2018.

\bibitem[Heusel et~al.(2017)Heusel, Ramsauer, Unterthiner, Nessler, and Hochreiter]{FID}
Heusel, M., Ramsauer, H., Unterthiner, T., Nessler, B., and Hochreiter, S.
\newblock Gans trained by a two time-scale update rule converge to a local nash equilibrium.
\newblock In \emph{NeurIPS}, 2017.

\bibitem[Hinton et~al.(2014)Hinton, Vinyals, and Dean]{hinton2015distilling}
Hinton, G., Vinyals, O., and Dean, J.
\newblock Distilling the knowledge in a neural network.
\newblock In \emph{NeurIPS-W}, 2014.

\bibitem[Ho et~al.(2020)Ho, Jain, and Abbeel]{ddpm}
Ho, J., Jain, A., and Abbeel, P.
\newblock Denoising diffusion probabilistic models.
\newblock In \emph{NeurIPS}, 2020.

\bibitem[Hooker et~al.(2020)Hooker, Moorosi, Clark, Bengio, and Denton]{hooker2020bias}
Hooker, S., Moorosi, N., Clark, G., Bengio, S., and Denton, E.
\newblock Characterising bias in compressed models.
\newblock \emph{arXiv preprint arXiv: 2010.03058}, 2020.

\bibitem[Hubara et~al.(2021)Hubara, Chmiel, Island, Banner, Naor, and Soudry]{transposable}
Hubara, I., Chmiel, B., Island, M., Banner, R., Naor, J., and Soudry, D.
\newblock Accelerated sparse neural training: A provable and efficient method to find n:m transposable masks.
\newblock In \emph{NeurIPS}, 2021.

\bibitem[Ioffe \& Szegedy(2015)Ioffe and Szegedy]{batchnorm}
Ioffe, S. and Szegedy, C.
\newblock Batch normalization: Accelerating deep network training by reducing internal covariate shift.
\newblock In \emph{ICML}, 2015.

\bibitem[Jordao et~al.(2020)Jordao, Lie, and Schwartz]{discriminative}
Jordao, A., Lie, M., and Schwartz, W.~R.
\newblock Discriminative layer pruning for convolutional neural networks.
\newblock \emph{IEEE Journal of Selected Topics in Signal Processing}, 2020.

\bibitem[Kim et~al.(2023)Kim, Jeong, Lee, and Song]{kim23efficient}
Kim, J., Jeong, Y., Lee, D., and Song, H.~O.
\newblock Efficient latency-aware cnn depth compression via two-stage dynamic programming.
\newblock In \emph{ICML}, 2023.

\bibitem[Krizhevsky(2009)]{cifar10}
Krizhevsky, A.
\newblock Learning multiple layers of features from tiny images.
\newblock \emph{Master's thesis, Department of Computer Science, University of Toronto}, 2009.

\bibitem[Krizhevsky et~al.(2012)Krizhevsky, Sutskever, and Hinton]{alexnet}
Krizhevsky, A., Sutskever, I., and Hinton, G.~E.
\newblock Imagenet classification with deep convolutional neural networks.
\newblock In \emph{NeurIPS}, 2012.

\bibitem[Li et~al.(2017)Li, Kadav, Durdanovic, Samet, and Graf]{pfec}
Li, H., Kadav, A., Durdanovic, I., Samet, H., and Graf, H.~P.
\newblock Pruning filters for efficient convnets.
\newblock In \emph{ICLR}, 2017.

\bibitem[Lin et~al.(2022)Lin, Kim, and Joo]{lin2022fairgrape}
Lin, X.-Z., Kim, S., and Joo, J.
\newblock Fairgrape: Fairness-aware gradient pruning method for face attribute classification.
\newblock In \emph{ECCV}, 2022.

\bibitem[{Liu} et~al.(2019){Liu}, {Mu}, {Zhang}, {Guo}, {Yang}, {Cheng}, and {Sun}]{metapruning}
{Liu}, Z., {Mu}, H., {Zhang}, X., {Guo}, Z., {Yang}, X., {Cheng}, K., and {Sun}, J.
\newblock Metapruning: Meta learning for automatic neural network channel pruning.
\newblock In \emph{ICCV}, 2019.

\bibitem[Liu et~al.(2022)Liu, Mao, Wu, Feichtenhofer, Darrell, and Xie]{convnext}
Liu, Z., Mao, H., Wu, C.-Y., Feichtenhofer, C., Darrell, T., and Xie, S.
\newblock A convnet for the 2020s.
\newblock In \emph{CVPR}, 2022.

\bibitem[Molchanov et~al.(2016)Molchanov, Tyree, Karras, Aila, and Kautz]{taylor16}
Molchanov, P., Tyree, S., Karras, T., Aila, T., and Kautz, J.
\newblock Pruning convolutional neural networks for resource efficient inference.
\newblock \emph{arXiv preprint arXiv:1611.06440}, 2016.

\bibitem[Molchanov et~al.(2019)Molchanov, Mallya, Tyree, Frosio, and Kautz]{taylor17}
Molchanov, P., Mallya, A., Tyree, S., Frosio, I., and Kautz, J.
\newblock Importance estimation for neural network pruning.
\newblock In \emph{CVPR}, June 2019.

\bibitem[Nichol \& Dhariwal(2021)Nichol and Dhariwal]{iddpm}
Nichol, A.~Q. and Dhariwal, P.
\newblock Improved denoising diffusion probabilistic models.
\newblock In \emph{ICML}, 2021.

\bibitem[Paganini(2020)]{paganini2020prune}
Paganini, M.
\newblock Prune responsibly.
\newblock \emph{arXiv preprint arXiv: 2009.09936}, 2020.

\bibitem[Paszke et~al.(2017)Paszke, Gross, Chintala, Chanan, Yang, DeVito, Lin, Desmaison, Antiga, and Lerer]{pytorch}
Paszke, A., Gross, S., Chintala, S., Chanan, G., Yang, E., DeVito, Z., Lin, Z., Desmaison, A., Antiga, L., and Lerer, A.
\newblock Automatic differentiation in pytorch.
\newblock In \emph{NeurIPS-W}, 2017.

\bibitem[Ronneberger et~al.(2015)Ronneberger, Fischer, and Brox]{unet}
Ronneberger, O., Fischer, P., and Brox, T.
\newblock U-net: Convolutional networks for biomedical image segmentation.
\newblock In \emph{MICCAI}, 2015.

\bibitem[Russakovsky et~al.(2015)Russakovsky, Deng, Su, Krause, Satheesh, Ma, Huang, Karpathy, Khosla, Bernstein, et~al.]{imagenet}
Russakovsky, O., Deng, J., Su, H., Krause, J., Satheesh, S., Ma, S., Huang, Z., Karpathy, A., Khosla, A., Bernstein, M., et~al.
\newblock Imagenet large scale visual recognition challenge.
\newblock In \emph{IJCV}, 2015.

\bibitem[Sandler et~al.(2018)Sandler, Howard, Zhu, Zhmoginov, and Chen]{mobilenetv2}
Sandler, M., Howard, A., Zhu, M., Zhmoginov, A., and Chen, L.-C.
\newblock Mobilenetv2: Inverted residuals and linear bottlenecks.
\newblock In \emph{CVPR}, 2018.

\bibitem[Shen et~al.(2022)Shen, Yin, Molchanov, Mao, Liu, and Alvarez]{halp}
Shen, M., Yin, H., Molchanov, P., Mao, L., Liu, J., and Alvarez, J.
\newblock Structural pruning via latency-saliency knapsack.
\newblock In \emph{NeurIPS}, 2022.

\bibitem[Song et~al.(2021)Song, Meng, and Ermon]{ddim}
Song, J., Meng, C., and Ermon, S.
\newblock Denoising diffusion implicit models.
\newblock In \emph{ICLR}, 2021.

\bibitem[Tran et~al.(2022)Tran, Fioretto, Kim, and Naidu]{tran2022pruning}
Tran, C., Fioretto, F., Kim, J.-E., and Naidu, R.
\newblock Pruning has a disparate impact on model accuracy.
\newblock In \emph{NeurIPS}, 2022.

\bibitem[Vanholder(2016)]{tensorrt}
Vanholder, H.
\newblock Efficient inference with tensorrt.
\newblock In \emph{GTC}, 2016.

\bibitem[Wightman et~al.(2021)Wightman, Touvron, and J{\'e}gou]{timm}
Wightman, R., Touvron, H., and J{\'e}gou, H.
\newblock Resnet strikes back: An improved training procedure in timm.
\newblock In \emph{NeurIPS-W}, 2021.

\bibitem[Wu \& He(2018)Wu and He]{groupnorm}
Wu, Y. and He, K.
\newblock Group normalization.
\newblock In \emph{ECCV}, 2018.

\end{thebibliography}
\bibliographystyle{icml2024}

%%%%%%%%%%%%%%%%%%%%%%%%%%%%%%%%%%%%%%%%%%%%%%%%%%%%%%%%%%%%%%%%%%%%%%%%%%%%%%%
%%%%%%%%%%%%%%%%%%%%%%%%%%%%%%%%%%%%%%%%%%%%%%%%%%%%%%%%%%%%%%%%%%%%%%%%%%%%%%%
% APPENDIX
%%%%%%%%%%%%%%%%%%%%%%%%%%%%%%%%%%%%%%%%%%%%%%%%%%%%%%%%%%%%%%%%%%%%%%%%%%%%%%%
%%%%%%%%%%%%%%%%%%%%%%%%%%%%%%%%%%%%%%%%%%%%%%%%%%%%%%%%%%%%%%%%%%%%%%%%%%%%%%%
\newpage
\appendix
\onecolumn
\section{Implementation Details} \label{app:detail}
In this section, we provide the implementation details regarding the different types of layers and networks.

\paragraph{Normalization layers}
Both ResNet-34 and MobileNetV2 utilize a batch normalization layer for normalization \citep{batchnorm}.
At inference time, we fuse the batch normalization layer with the convolution layer, before merging the network.
On the other hand, DDPM employs a group normalization layer for normalization \citep{groupnorm}.
Unlike batch normalization, this layer cannot be fused with the convolution layer at inference time because it uses test-time feature map statistics.
To address this, before fine-tuning, we move any group normalization layer between successive convolution layers that will be merged, to after these convolution layers.
If there are multiple such normalization layers, we only keep the last one.
This adjustment allows us to merge the consecutive convolution layers at inference time.

\paragraph{Strided convolutions, depthwise convolutions, and padding}

For convolution layers that have a stride larger than 1, we restrict merging them when the kernel size of the following convolution layer is larger than 1, as done in \citet{kim23efficient}. We thus restrict the choice of $A$ to include the activation layers between such layers.
This is because merging the convolution layer with a stride larger than 1 significantly increases the merged kernel size.
Concretely, merging the consecutive convolution layers $f_{\theta_2} \circ f_{\theta_1}$ results in the merged kernel size of
\begin{equation*}
    \mathrm{Ker}(\theta_{\mathrm{merged}}) = (\mathrm{Ker}(\theta_2)-1) \times \mathrm{Str}(f_{\theta_1}) +  \mathrm{Ker}(\theta_1),
\end{equation*}
where $\mathrm{Str}(\cdot)$ denotes the stride of the convolution layer \citep{depthshrinker}.

It is worth mentioning that merging two depthwise convolution layers results in a single depthwise convolution layer. 
To account for this in networks that include depthwise convolution layers, we add a binary variable to the lookup table to indicate whether a layer is a depthwise convolution layer.
We integrate this variable into the dynamic programming algorithm when we implement our method.

As noted in \citet{kim23efficient}, consecutive convolution layers that will be merged should not have any padding in between to avoid discrepancies at the boundary of the output before and after merging.
%to be merged exactly at the test time.
To address this, before fine-tuning, we reorder the padding of the network to be applied prior to each such block of consecutive convolution layers.
%after we obtain the optimal sets through \Cref{alg:dp}, following the practice of 
as done in \citet{kim23efficient}.

\paragraph{Skip-connections}
There are two different types of skip-connection in a CNN: Skip-addition and skip-concatenation.
Skip-addition, employed by ResNet and MobileNetV2, 
adds the output of an earlier layer directly to the output of a later layer.
%involves the addition of the input to the output of successive convolution layers.
We can fuse skip-addition into the convolution layer if every intermediate convolution layer is merged into a single layer \citep{repvgg}.
We thus can merge across a skip-addition only if every convolution layer in between is merged into a single layer. % and we can fuse skip-addition inside the convolution weight. 
Conversely, skip-concatenation, used in DDPM, concatenates the output of an earlier layer directly to the output of a later layer.
%the current input with the output of subsequent convolution layers.
We do not merge layers across a skip-concatenation. %since it is not straightforward to merge it as a single convolution layer.

\paragraph{MobileNetV2}
In MobileNetV2, it is notable that there is no non-linear activation layer following each Inverted Residual Block \citep{mobilenetv2}. 
Prior works on depth compression suggest that the performance of a compressed network can be improved by adding a non-linear activation layer after the merged layers \citep{depthshrinker,kim23efficient}. 
We also adapt this trick in our implementation on MobileNetV2 architecture.

\paragraph{DDPM}
The DDPM architecture uses an upsampling module to increase the spatial dimension of the feature map, and it further employs a self-attention layer at the $16 \times 16$ resolution between the residual blocks \citep{chen2018pixelsnail,ddpm}. 
We do not merge convolution layers across the self-attention layer or the upsampling layer in the DDPM architecture.
It is also worth mentioning that the DDPM architecture has a $3\times3$ convolution layer with stride $1$  between the upsampling layer and the skip concatenation.
We include these convolution layers as potential pruning candidates in our algorithm, as this leads to improved performance of the compressed network.
When measuring $\mathrm{Perf}(\cdot)$ using negative diffusion loss in DDPM, we further divide this value by the diffusion loss of the pre-trained network since it leads to more stable result.

\section{Proof} \label{app:proof}
In this section, we prove \Cref{prop:dp}, restated here for convenience.
\DPoptimal*
\begin{proof}
We prove this by induction.
Suppose that for $l < l_0$ and $t < t_0$, $A[l, t]$ and $k_i^{(lt)} \coloneqq C[l, t] \cap (a_{i-1}^{(lt)}, a_{i}^{(lt)}]$ are the optimal solution of Problem \eqref{eq:dp_subprob}, where $(a_{i}^{(lt)})_{i=1}^{|A[l,t]|}$ denotes the element of $A[l, t]$ in ascending order.

Assume that $A[l_0, t_0]$ and $k_i^{(l_0 t_0)}$ from \Cref{eq:dp_recur} is not the optimal solution of Problem \eqref{eq:dp_subprob} when $l=l_0, t=t_0$. 
We denote the optimal solution of Problem \eqref{eq:dp_subprob} for $l=l_0, t=t_0$ as $\widehat{A}$ and $\hat{k}_i$.

\begin{align} 
    \sum_{i=1}^{|\widehat{A}|+1} I\bigl[\hat{a}_{i-1}, \hat{a}_i, \hat{k}_i\bigr] &> \sum_{i=1}^{|A[l_0, t_0]|+1} I\bigl[a_{i-1}^{(l_0 t_0)}, a_{i}^{(l_0 t_0)}, k_i^{(l_0 t_0)}\bigr], \label{prf:eq1}\\
    \sum_{i=1}^{|\widehat{A}|+1} T\bigl[\hat{a}_{i-1}, \hat{a}_i, \hat{k}_i\bigr] \label{prf:eq2}&< t_0. 
\end{align}

We first show that $\widehat{A}$ is not empty because 
\begin{align*} %\label{prf:not_empty}
    &\sum_{i=1}^{|\widehat{A}|+1} I\bigl[\hat{a}_{i-1}, \hat{a}_i, \hat{k}_i\bigr] \\&> \sum_{i=1}^{|A[l_0, t_0]|+1} I\bigl[a_{i-1}^{(l_0t_0)}, a_{i}^{(l_0t_0)}, k_i^{(l_0 t_0)}\bigr]= M[l_0, t_0] \tag{from the assumption \& definition of $A[l_0, t_0]$}\\
    &= \underset{0 \leq l' < l_0,\,k \in K_{l' l_0}}{\mathrm{max}} \left(M[l', t - T[l', l_0, k]] + I[l', l_0, k]\right) \tag{from the recurrence relation \Cref{eq:dp_recur}}\\
    &\geq \underset{k \in K_{0 l_0}}{\mathrm{max}} I[0, l_0, k].
\end{align*}

Now, let $\hat{l} \coloneqq \hat{a}_{|\widehat{A}|}$ be the maximum value of $\widehat{A}$.
Then, 
\begin{align}
    &\sum_{i=1}^{|\widehat{A}|+1} I\bigl[\hat{a}_{i-1}, \hat{a}_i, \hat{k}_i\bigr] \nonumber \\
    &= \sum_{i=1}^{|\widehat{A}|} I\bigl[\hat{a}_{i-1}, \hat{a}_{i}, \hat{k}_i\bigr] + I[\hat{l}, l_0, \hat{k}_{|\widehat{A}|+1}] \nonumber\\
    &\leq M[\hat{l}, t_0 - T[\hat{l}, l_0, \hat{k}_{|\widehat{A}|+1}]] + I[\hat{l}, l_0, \hat{k}_{|\widehat{A}|+1}] \tag{from the optimality assumption} \nonumber \\
    &\leq M[l_0, t_0] = \sum_{i=1}^{|A[l_0, t_0]|+1} I\bigl[a_{i-1}^{(l_0 t_0)}, a_{i}^{(l_0 t_0)}, k_i^{(l_0 t_0)}\bigr] \tag{from the recurrence relation \Cref{eq:dp_recur}}, 
\end{align}
which contradicts to \Cref{prf:eq1}.
Note that the inequality with from the optimality assumption holds because we assumed the optimality for $l=\hat{l}<l_0$ and $t=t_0-T[\hat{l}, l_0, \hat{k}_{|\widehat{A}|+1}] < t_0$,
and from \Cref{prf:eq2}, we have
$$ \sum_{i=1}^{|\widehat{A}|} T\bigl[\hat{a}_{i-1}, \hat{a}_i, \hat{k}_i\bigr] <  t_0-T[\hat{l}, l_0, \hat{k}_{|\widehat{A}|+1}],$$ which satisfies the constraint in Problem \eqref{eq:dp_subprob}.
Therefore, from the contradiction we have that $A[l_0, t_0]$ and $k_i^{(l_0 t_0)}$ are indeed optimal solutions of Problem \eqref{eq:dp_subprob} ($l=l_0$ and $t=t_0$).

For the base case $l=0$ and $t = \frac{T_0}{P}$, $A[l, t] = \emptyset$ is indeed the solution of Problem \eqref{eq:dp_subprob}.

Now, plugging in $l=L$ and $t=T_0$ proves the theorem.
\end{proof}

\section{Details on Constructing Importance and Latency Tables} \label{app:cost}

In this section, we provide details on how we measure the importance and latency values for the lookup tables $I$ and $T$, along with their corresponding practical computation cost for different types of networks.

\begin{table}[t]
\caption{Wall-clock time for constructing the importance and latency look-up tables. GPU hours for constructing the importance table is measured in RTX3090 and the latency table is measured in RTX2080 Ti.}
\vspace{0.5em}
\centering
\begin{adjustbox}{max width=1.0\columnwidth}
\begin{tabular}{lcccc}
\toprule
 & & Importance table & Latency table & \# of table \\
 Network    & Dataset & (GPU Hours) & (GPU Hours) & entries  \\
\cmidrule(r){1-1}\cmidrule(r){2-2}\cmidrule(r){3-3}\cmidrule(r){4-4}\cmidrule(r){5-5}
    ResNet-34        & ImageNet &  4.4 hours & 25.9 minutes & 150 \\
    MobileNetV2-1.0  & ImageNet & 13.2 hours &  6.2 minutes & 391 \\
    MobileNetV2-1.4  & ImageNet & 15.0 hours & 10.6 minutes & 391\\
% \cmidrule(r){1-1}\cmidrule(r){2-2}\cmidrule(r){3-3}\cmidrule(r){4-4}
    DDPM             & CIFAR10  &  2.5 hours &  1.3 minutes & 98\\    
\bottomrule
\end{tabular}
\end{adjustbox}
\label{tab:cost}
\end{table}

\begin{table}[t]
\caption{Wall-clock time for constructing the importance look-up table for ResNet-34 and MobileNetV2-1.0 in different methods. GPU hours for constructing the importance table is measured in RTX3090.}
\centering
\begin{subtable}{0.47\textwidth}

\caption{ResNet-34 on ImageNet dataset.}
\begin{adjustbox}{max width=\columnwidth}
\begin{tabular}{lcc}
\toprule
 & &  \# of table \\
 Method     & GPU Hours & entries  \\
\cmidrule(r){1-1}\cmidrule(r){2-2}\cmidrule(r){3-3}
    Depth \citep{kim23efficient} &  25.8 hours & 62 \\
    LayerOnly (Ours) & 0.8 hours & 29 \\
    LayerMerge (Ours) & 4.4 hours & 150 \\
\bottomrule
\end{tabular}
\end{adjustbox}
\end{subtable}
\hfill
\begin{subtable}{0.47\textwidth}
\caption{MobileNetV2-1.0 on ImageNet dataset.}
\begin{adjustbox}{max width=\columnwidth}
\begin{tabular}{lcc}
\toprule
 & & \# of table \\
 Method    & GPU Hours & entries  \\
\cmidrule(r){1-1}\cmidrule(r){2-2}\cmidrule(r){3-3}
    Depth \citep{kim23efficient} &  126.0 hours & 315 \\
    LayerOnly (Ours) & 0.4 hours & 13 \\
    LayerMerge (Ours) & 13.2 hours & 391 \\ 
\bottomrule
\end{tabular}
\end{adjustbox}
\end{subtable}
\label{tab:cost-cmp}
\end{table}

\paragraph{Importance measurements}

We defined importance values in \Cref{eq:imp}.
Recall that the first term is estimated with the performance of the network after fine-tuning for a few steps.
For that, we select a random subset of the training dataset for fine-tuning, then we evaluate performance on another separate subset, also drawn from the training dataset.
% on another separate subset used as a holdout validation set.
For the second term, we evaluate the performance of the pre-trained network on the separate subset.
In particular, we use a fine-tuning subset of size 4\% of the total training dataset size for ImageNet, and 1\% for CIFAR10. The separate subset is also the same size as the fine-tuning subset. 
We fine-tune the network for 1 epoch for ImageNet, and 50 epochs for CIFAR10 with the fine-tuning subset when we measure the importance.

% for the ImageNet dataset, we use a random subset of size 4\% of the total training dataset size for fine-tuning and fine-tune for 1 epoch. For the CIFAR10 dataset, the subset size is set to be 1\% of the total training dataset size.
% Subsequently, for ImageNet dataset, we fine-tune the network for the 1 epoch on the random subset when we measure the importance.
% For CIFAR10 dataset, we fine-tune the network for the 50 epochs on the random subset.
We report the wall-clock time for constructing the importance look-up table in \Cref{tab:cost}.
It is worth mentioning that this computation can be done in an embarrassingly parallel fashion without any communication between GPUs, which allows for significant speedup if multiple GPUs are available. 
For instance, the importance table for MobileNetV2-1.0 only took 33 minutes with 24 GPUs.

In \Cref{tab:cost-cmp}, we further compare the wall-clock time required to construct the importance look-up tables used in LayerMerge and LayerOnly to the one used in Depth \citep{kim23efficient}. 
It is worth noting that the Depth baseline fine-tunes for one epoch over the full training dataset to evaluate each table entry.
However, we observe that fine-tuning for one epoch using only a small random subset of the training dataset is sufficient for estimating the importance values.
This significantly reduces the wall-clock time required to construct the look-up table in our method compared to the Depth baseline.

\paragraph{Latency measurements}
To measure each latency value in the lookup table $T$, we measure the inference time on PyTorch by first warming up the GPU for 300 forward passes, then averaging the inference time over the subsequent 200 forward passes, with a batch size of 128. Latency values are measured in milliseconds.
We report the wall-clock time for constructing the latency look-up table in \Cref{tab:cost}.

Recall that to solve the surrogate problem, we first need to discretize the latency values.
We do that by multiplying the real valued latencies in the lookup table and $T_0$ by $10$, then rounding them down to the nearest integer. 
Note that this is equivalent to choosing the discretization level as $P = 10 T_0$.

\section{Details on Ablation Studies} \label{app:exp}

In this section, we outline the details of the sequential optimization baseline (\emph{Depth $\rightarrow$ LayerOnly}) presented in \Cref{tab:abl}.
Recall that for MobileNetV2, we fine-tune every compressed network for 180 epochs, using the same fine-tuning recipe as \citet{kim23efficient}. %We first note that we report the results of fine-tuning the networks for 180 epochs each following the training recipe from \citet{kim23efficient}.
For the sequential optimization baseline, we divide the fine-tuning epochs equally between the two pruning methods, i.e., 
we fine-tune for 90 epochs after each pruning method, again using the same fine-tuning recipe.

We fix the latency budget ratio $p\%$ for {Depth} to $74\%$, which yields a speed-up of 1.63$\times$. Then we use two different values of $p$ for {LayerOnly}:  $72\%$ and $64\%$. Note that $p$ here corresponds to $T_0/T_{\text{depth-pruned}}$, where $T_0$ is the chosen latency budget of the final pruned model and $T_{\text{depth-pruned}}$ is the latency of the model pruned using  {Depth-74\%}. 
% Specifically, we first compress the network with depth compression baseline and fine-tune it for 90 epochs \citep{kim23efficient}.
% Subsequently, we bring the network with 1.63$\times$ speed-up from depth compression and further compress it with a different compression ratio using the \textit{LayerOnly} variant of our method.
It is worth noting that the allocations of fine-tuning epochs and compression ratios between the two pruning methods are hyperparameters that need to be tuned. 
However, our method is free from these hyperparameters due to the joint optimization on the two types of pruning modalities. %the two different set variables.

\section{Additional Experiments}\label{app:AddExps}

\paragraph{Results with smaller fine-tuning epochs}
\looseness=-1 In this section, we study the effect of fine-tuning for a shorter time.
In particular, we present in \Cref{tab:small-mbv2} compression results on MobileNetV2-1.0 where we fine-tune all methods for 90, 30, and 20 epochs, using cosine learning rate decay. 
We further plot in \Cref{fig:recovery} the recovery curve of test accuracy across fine-tuning steps. % when we fine-tune for 180 epochs.
Our method consistently outperforms baselines under these smaller fine-tuning budgets as well.

\begin{table}
\caption{Accuracy and latency speed-up of applying compression methods to MobileNetV2-1.0 on ImageNet dataset with 90, 30, and 20 fine-tuning epochs. The latency speed-up is measured on RTX2080 Ti GPU at batch size 128.}
\begin{subtable}{0.31\textwidth}

\caption{Results of fine-tuning for 90 epochs.}

\centering
\begin{adjustbox}{max width=1.0\columnwidth}
\begin{tabular}{lcccc}
\toprule
 && \normalsize{PyTorch}& \normalsize{TensorRT}\\
 Network     & Acc (\%) $\uparrow$& Speed-up $\uparrow$ & Speed-up $\uparrow$\\
 \cmidrule(r){1-2}\cmidrule(r){3-3} \cmidrule(r){4-4}
    MobileNetV2-1.0  & 72.89 & 1.00$\times$ & 1.00$\times$ \\
 \cmidrule(r){1-2}\cmidrule(r){3-3} \cmidrule(r){4-4}
    AMC-70\% \citep{amc} & 71.66 & 1.32$\times$ & 1.34$\times$ \\
    Depth-74\% \citep{kim23efficient} & 72.49 & 1.62$\times$ & \textbf{1.42}$\times$ \\
    LayerOnly-73\% (Ours) & 69.29 & 1.30$\times$ & 1.35$\times$\\
    LayerMerge-55\% (Ours)& \textbf{72.73}  & \textbf{1.63}$\times$ & \textbf{1.42}$\times$ \\
\bottomrule
\end{tabular}
\end{adjustbox}
\end{subtable}
\hfill
\begin{subtable}{0.31\textwidth}

\caption{Results of fine-tuning for 30 epochs.}
\centering
\begin{adjustbox}{max width=1.0\columnwidth}
\begin{tabular}{lcccc}
\toprule
 && \normalsize{PyTorch}& \normalsize{TensorRT}\\
 Network     & Acc (\%) $\uparrow$& Speed-up $\uparrow$ & Speed-up $\uparrow$\\
 \cmidrule(r){1-2}\cmidrule(r){3-3} \cmidrule(r){4-4}
    MobileNetV2-1.0  & 72.89 & 1.00$\times$ & 1.00$\times$ \\
 \cmidrule(r){1-2}\cmidrule(r){3-3} \cmidrule(r){4-4}
    AMC-70\% \citep{amc} & 71.05 & 1.32$\times$ & 1.34$\times$ \\
    Depth-74\% \citep{kim23efficient} & 71.59 & 1.62$\times$ & \textbf{1.42}$\times$ \\
    LayerOnly-73\% (Ours) & 67.60 & 1.30$\times$ & 1.35$\times$\\
    LayerMerge-55\% (Ours)& \textbf{72.06}  & \textbf{1.63}$\times$ & \textbf{1.42}$\times$ \\
\bottomrule
\end{tabular}
\end{adjustbox}
\end{subtable}
\hfill
\begin{subtable}{0.31\textwidth}

\caption{Results of fine-tuning for 20 epochs.}
\centering
\begin{adjustbox}{max width=1.0\columnwidth}
\begin{tabular}{lcccc}
\toprule
 && \normalsize{PyTorch}& \normalsize{TensorRT}\\
 Network     & Acc (\%) $\uparrow$& Speed-up $\uparrow$ & Speed-up $\uparrow$\\
 \cmidrule(r){1-2}\cmidrule(r){3-3} \cmidrule(r){4-4}
    MobileNetV2-1.0  & 72.89 & 1.00$\times$ & 1.00$\times$ \\
 \cmidrule(r){1-2}\cmidrule(r){3-3} \cmidrule(r){4-4}
    AMC-70\% \citep{amc} & 70.64 & 1.32$\times$ & 1.34$\times$ \\
    Depth-74\% \citep{kim23efficient} & 71.09 & 1.62$\times$ & \textbf{1.42}$\times$ \\
    LayerOnly-73\% (Ours) & 66.96 & 1.30$\times$ & 1.35$\times$\\
    LayerMerge-55\% (Ours)& \textbf{71.59}  & \textbf{1.63}$\times$ & \textbf{1.42}$\times$ \\
\bottomrule
\end{tabular}
\end{adjustbox}
\end{subtable}

\label{tab:small-mbv2}
% \vspace{1em}
\end{table}

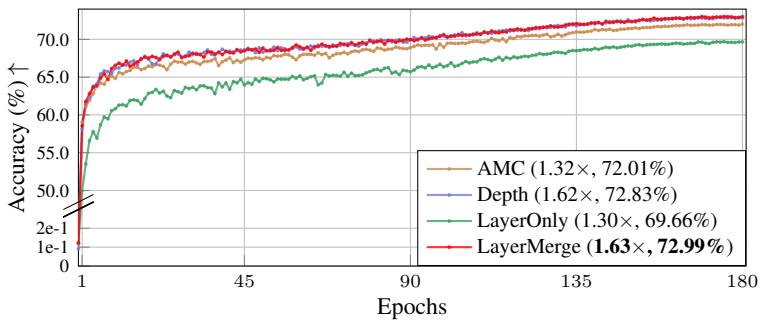
\begin{figure}
\centering
\begin{tikzpicture}[define rgb/.code={\definecolor{mycolor}{RGB}{#1}},
                    rgb color/.style={define rgb={#1},mycolor}]

\definecolor{gr}{RGB}{60,160,100}
\definecolor{or}{RGB}{200,140,80}
\definecolor{bl}{RGB}{120,120,220}
\definecolor{yl}{RGB}{200,200,100}
\definecolor{pp}{RGB}{200,150,240}

\begin{groupplot}[
        group style={columns=1, horizontal sep=1.05cm, 
        vertical sep=0.0cm},
        ]

\nextgroupplot[
            % Figure size
            width=0.61\columnwidth,
            height=5cm,
            % Plot style
            every axis plot/.append style={thick},
            % Grid
            grid=major,
            % Tick
            scaled ticks = false,
            % xmajorticks=false,
            xlabel near ticks,
            ylabel near ticks,
            tick pos=left,
            tick label style={font=\scriptsize},
            % Label
            xlabel shift=-0.1cm,         
            ylabel shift=-0.15cm,
            label style={font=\small},
            xlabel style={align=center},
            xlabel={Epochs},
            ylabel={Accuracy (\%) $\uparrow$},
            % Range
            xmin=0,
            xmax=181,
            xtick={1, 45, 90,135, 180},
            ytick={70.0, 65.0, 60.0, 55.0, 50.0, 45.0, 42.5, 40.0},
            yticklabels={70.0, 65.0, 60.0, 55.0, 50.0, 2e-1, 1e-1, 0},
            ymin=40,
            ymax=74.0,
            %legend to name=grouplegend,
            legend cell align=left,
            legend style={at={(1,0)},anchor=south east, nodes={scale=0.8}},
            after end axis/.code={
         \draw (rel axis cs:0,0.22) +(-2mm,-1mm) -- +(2mm,1mm)
              ++(0pt,-\pgfkeysvalueof{/tikz/axis break gap})
              +(-2mm,-1mm) -- +(2mm,1mm)
              (rel axis cs:0,0.22) +(0mm,0mm) -- +(0mm,0mm)
              ++(0pt,-\pgfkeysvalueof{/tikz/axis break gap})
              +(-2mm,-1mm) -- +(2mm,1mm);
              \draw (rel axis cs:0,0.22) +(-2mm,0mm) -- +(2mm,2mm)
              ++(0pt,-\pgfkeysvalueof{/tikz/axis break gap})
              +(-2mm,-1mm) -- +(2mm,1mm)
              (rel axis cs:0,0.22) +(0mm,0mm) -- +(0mm,0mm)
              ++(0pt,-\pgfkeysvalueof{/tikz/axis break gap})
              +(-2mm,-1mm) -- +(2mm,1mm);
    }
    ]

% \addplot[gr, opacity=0.8] table [y=van, col sep=comma]{data/fine-tune.csv};\addlegendentry{Vanilla (40.84ms)}

\addplot[or, opacity=0.8, mark=*, mark size=0.3pt] table [y=channel, col sep=comma]{data/recovery_with_zero.csv};\addlegendentry{AMC (1.32$\times$, 72.01\%)}

\addplot[bl, opacity=0.8, mark=*, mark size=0.3pt] table [y=depth, col sep=comma]{data/recovery_with_zero.csv};\addlegendentry{Depth (1.62$\times$, 72.83\%)}

\addplot[gr, opacity=0.8, mark=*, mark size=0.3pt] table [y=layer, col sep=comma]{data/recovery_with_zero.csv};\addlegendentry{LayerOnly (1.30$\times$, 69.66\%)}

\addplot[red, opacity=0.8, mark=*, mark size=0.3pt] table [y=layer_merge, col sep=comma]{data/recovery_with_zero.csv};\addlegendentry{LayerMerge (\textbf{1.63$\times$, 72.99\%})}
\end{groupplot}
\end{tikzpicture}
    \caption{
    Test accuracy recovery curve of different compression methods across fine-tuning epochs. We indicate the associated speed-up and accuracy after fine-tuning in the parentheses. The inference time is measured on RTX2080 Ti GPU at batch size 128 in PyTorch format.}
    \label{fig:recovery}
    \vspace{-1em}
\end{figure}

\paragraph{Comparison with knowledge distillation}
In this section, we compare our method to the knowledge distillation method of \citet{hinton2015distilling}. 
For that, we use a smaller version of MobileNetV2 \citep{mobilenetv2} than the one used for the pre-trained network as the student network and train it for the same number of epochs we use for fine-tuning in our method (180 epochs).
We present the results in \Cref{tab:kd1-mbv2} and plot the recovery curve of test accuracy across fine-tuning steps in \Cref{fig:kd-recovery}.
The key benefit of pruning methods like ours over knowledge distillation is that they only require fine-tuning the model, while knowledge distillation requires training the small model from scratch. 
This provides an advantage when both methods are compared under an identical training budget. 
Indeed our results show that our method outperforms the knowledge distillation method in this setting.
%The results show that our method outperforms the knowledge distillation method.

\begin{table}[t]
\caption{Accuracy and latency speed-up comparison between knowledge distillation \citep{hinton2015distilling} and our method with MobileNetV2-1.0 and MobileNetV2-1.4 on ImageNet dataset. The latency speed-up is measured on RTX2080 Ti GPU at batch size 128.}
\begin{subtable}{0.47\textwidth}

\caption{MobileNetV2-1.0 on ImageNet dataset.}

\centering
\begin{adjustbox}{max width=1.0\columnwidth}
\begin{tabular}{lcccc}
\toprule
 && \normalsize{PyTorch}& \normalsize{TensorRT}\\
 Network     & Acc (\%) $\uparrow$& Speed-up $\uparrow$ & Speed-up $\uparrow$\\
 \cmidrule(r){1-2}\cmidrule(r){3-3} \cmidrule(r){4-4}
    MobileNetV2-1.0  & 72.89 & 1.00$\times$ & 1.00$\times$ \\
 \cmidrule(r){1-2}\cmidrule(r){3-3} \cmidrule(r){4-4}
    Knowledge distillation (MobileNetV2-0.75) & 69.69 & 1.17$\times$ & 1.20$\times$ \\
    LayerMerge-55\% (Ours)& \textbf{72.99}  & \textbf{1.63}$\times$ & \textbf{1.42}$\times$ \\
\bottomrule
\end{tabular}
\end{adjustbox}
\end{subtable}
\hfill
\begin{subtable}{0.47\textwidth}

\caption{MobileNetV2-1.4 on ImageNet dataset.}
\centering
\begin{adjustbox}{max width=1.0\columnwidth}
\begin{tabular}{lcccc}
\toprule
 && \normalsize{PyTorch}& \normalsize{TensorRT}\\
 Network     & Acc (\%) $\uparrow$& Speed-up $\uparrow$ & Speed-up $\uparrow$\\
 \cmidrule(r){1-2}\cmidrule(r){3-3} \cmidrule(r){4-4}
    MobileNetV2-1.4  & 76.28 & 1.00$\times$ & 1.00$\times$ \\
 \cmidrule(r){1-2}\cmidrule(r){3-3} \cmidrule(r){4-4}
    Knowledge distillation (MobileNetV2-1.0) & 72.30 & 1.51$\times$ & 1.54$\times$ \\
    LayerMerge-43\% (Ours)& \textbf{74.91}  & \textbf{1.99}$\times$ & \textbf{1.61}$\times$ \\
\bottomrule
\end{tabular}
\end{adjustbox}
\end{subtable}
\vspace{1em}
\label{tab:kd1-mbv2}
\end{table}

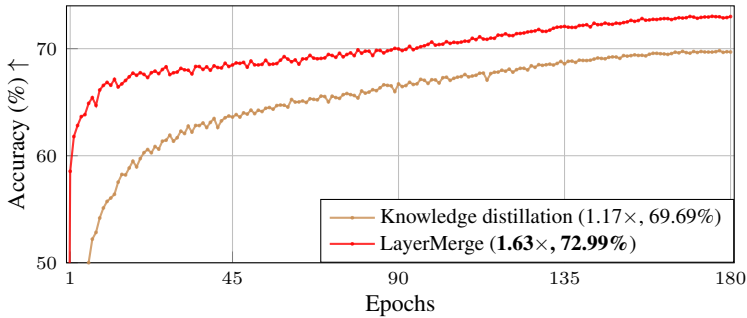
\begin{figure}[t]
\centering
\begin{tikzpicture}[define rgb/.code={\definecolor{mycolor}{RGB}{#1}},
                    rgb color/.style={define rgb={#1},mycolor}]

\definecolor{gr}{RGB}{60,160,100}
\definecolor{or}{RGB}{200,140,80}
\definecolor{bl}{RGB}{120,120,220}
\definecolor{yl}{RGB}{200,200,100}
\definecolor{pp}{RGB}{200,150,240}

\begin{groupplot}[
        group style={columns=1, horizontal sep=1.05cm, 
        vertical sep=0.0cm},
        ]

\nextgroupplot[
            % Figure size
            width=0.61\columnwidth,
            height=5cm,
            % Plot style
            every axis plot/.append style={thick},
            % Grid
            grid=major,
            % Tick
            scaled ticks = false,
            % xmajorticks=false,
            xlabel near ticks,
            ylabel near ticks,
            tick pos=left,
            tick label style={font=\scriptsize},
            % Label
            xlabel shift=-0.1cm,         
            ylabel shift=-0.15cm,
            label style={font=\small},
            xlabel style={align=center},
            xlabel={Epochs},
            ylabel={Accuracy (\%) $\uparrow$},
            % Range
            xmin=0,
            xmax=181,
            xtick={1, 45, 90,135, 180},
            ytick={70.0, 60.0, 50.0, 40.0, 30.0, 20.0, 10.0},
            ymin=50,
            ymax=74.0,
            %legend to name=grouplegend,
            legend cell align=left,
            legend style={at={(1,0)},anchor=south east, nodes={scale=0.8}}]

% \addplot[gr, opacity=0.8] table [y=van, col sep=comma]{data/fine-tune.csv};\addlegendentry{Vanilla (40.84ms)}

\addplot[or, opacity=0.8, mark=*, mark size=0.3pt] table [y=kd, col sep=comma]{data/recovery_kd.csv};\addlegendentry{Knowledge distillation (1.17$\times$, 69.69\%)}

% \addplot[bl, opacity=0.8, mark=*, mark size=0.3pt] table [y=depth, col sep=comma]{data/recovery_with_zero.csv};\addlegendentry{Depth (1.62$\times$, 72.83\%)}

% \addplot[gr, opacity=0.8, mark=*, mark size=0.3pt] table [y=scratch, col sep=comma]{data/recovery_kd_scratch.csv};\addlegendentry{Scratch (1.17$\times$, 69.69\%)}

\addplot[red, opacity=0.8, mark=*, mark size=0.3pt] table [y=layer_merge, col sep=comma]{data/recovery_kd.csv};\addlegendentry{LayerMerge (\textbf{1.63$\times$, 72.99\%})}
\end{groupplot}
\end{tikzpicture}
    \caption{
    Test accuracy recovery curve of our method compared to knowledge distillation across fine-tuning epochs for MobileNetV2-1.0. We indicate the associated speed-up and the accuracy after fine-tuning in the parentheses. The inference time on RTX2080 Ti GPU at batch size 128 in PyTorch format.}
    \label{fig:kd-recovery}
    \vspace{-1em}
\end{figure}

\paragraph{Applying knowledge distillation during fine-tuning}

It is worth noting that knowledge distillation methods can be jointly applied with pruning methods, by considering the pre-trained network as the teacher network and the pruned network as the student network to train. 
We present in \Cref{tab:kd2-mbv2-1.0} the results of applying the knowledge distillation method of \citet{hinton2015distilling} to different pruning methods. We observe that applying knowledge distillation further improves the accuracy of the pruned network (compared to \Cref{tab:mbv2-1.0}), and our method consistently outperforms the baselines in this setting as well.

\begin{table}[t]
\caption{Accuracy and latency speed-up of applying knowledge distillation \citep{hinton2015distilling} with compression methods to MobileNetV2-1.0 on ImageNet dataset. The latency speed-up is measured on RTX2080 Ti GPU at batch size 128.}
\vspace{0.5em}
\centering
\begin{adjustbox}{max width=0.5\columnwidth}
\begin{tabular}{lcccc}
\toprule
 && \normalsize{PyTorch}& \normalsize{TensorRT}\\
 Network     & Acc (\%) $\uparrow$& Speed-up $\uparrow$ & Speed-up $\uparrow$\\
 \cmidrule(r){1-2}\cmidrule(r){3-3} \cmidrule(r){4-4}
    MobileNetV2-1.0  & 72.89 & 1.00$\times$ & 1.00$\times$ \\
 \cmidrule(r){1-2}\cmidrule(r){3-3} \cmidrule(r){4-4}
    AMC-70\% \citep{amc} & 72.04 & 1.32$\times$ & 1.34$\times$ \\
    Depth-74\% \citep{kim23efficient} & 72.99 & 1.62$\times$ & \textbf{1.42}$\times$ \\
    LayerOnly-73\% (Ours) & 69.70 & 1.30$\times$ & 1.35$\times$\\
    LayerMerge-55\% (Ours)& \textbf{73.14}  & \textbf{1.63}$\times$ & \textbf{1.42}$\times$ \\
 \cmidrule(r){1-2}\cmidrule(r){3-3} \cmidrule(r){4-4}
    Depth-66\% \citep{kim23efficient} & 72.31 & 1.88$\times$ & 1.57$\times$ \\
    LayerMerge-46\% (Ours) & \textbf{72.56} & \textbf{1.90}$\times$ & \textbf{1.65}$\times$ \\
 \cmidrule(r){1-2}\cmidrule(r){3-3} \cmidrule(r){4-4}
    Depth-59\% \citep{kim23efficient} & 71.76 & 2.07$\times$ & 1.79$\times$ \\
    LayerMerge-38\% (Ours) & \textbf{72.06} & \textbf{2.18}$\times$ & \textbf{1.84}$\times$ \\
 \cmidrule(r){1-2}\cmidrule(r){3-3} \cmidrule(r){4-4}
    Depth-53\% \citep{kim23efficient} & 70.81 & 2.47$\times$ & 1.97$\times$ \\
    LayerMerge-33\% (Ours) & \textbf{71.32} & \textbf{2.49}$\times$ & \textbf{2.05}$\times$ \\
\bottomrule
\end{tabular}
\end{adjustbox}
\label{tab:kd2-mbv2-1.0}
\end{table}

\input{figures/pareto-row3-col2}
% \paragraph{}

\paragraph{Additional compression results}

In this section, we present additional results comparing various compression methods across different network architectures (ResNet-34, MobileNetV2-1.0, MobileNetV2-1.4, and DDPM) and compression ratios. 
%We explore a range of compression ratios and 
We display the Pareto curves for each method on the different architectures in \Cref{fig:pareto}.
We report the latency speed-up measured in PyTorch format.
For the ResNet and MobileNetV2 architectures, we plot accuracy against speed-up, and for the DDPM architecture, we plot the FID metric against the latency speed-up.

\end{document}